\documentclass[10pt]{article}

\usepackage{amsmath,amsfonts,bm}

\def\eqref#1{equation~\ref{#1}}
\def\1{\bm{1}}

\DeclareMathAlphabet{\mathsfit}{\encodingdefault}{\sfdefault}{m}{sl}
\SetMathAlphabet{\mathsfit}{bold}{\encodingdefault}{\sfdefault}{bx}{n}

\newcommand{\softmax}{\mathrm{softmax}}

\usepackage{multicol}
\usepackage{multirow}
\usepackage{url}
\usepackage{url}
\usepackage{booktabs}       % professional-quality tables
\usepackage{amsfonts}       % blackboard math symbols
\usepackage{nicefrac}       % compact symbols for 1/2, etc.
\usepackage{microtype}      % microtypography
\usepackage{xcolor}         % colors
\usepackage{graphicx}
\usepackage{wrapfig,lipsum}
\usepackage{multirow}
\usepackage{epsfig}
\usepackage{amsmath}
\usepackage{amssymb}
\usepackage{subfigure}
\usepackage{marvosym}
\usepackage{color}
\usepackage{threeparttable}
\usepackage{algorithm}
\usepackage{algpseudocode}
\usepackage{setspace}
\usepackage{amsthm}
\usepackage{verbatim}
\usepackage{adjustbox}
\usepackage{lineno}
\usepackage[title]{appendix}
\usepackage{xr} % For external references
\usepackage{float}

\usepackage[backref]{hyperref}
\usepackage[margin=1in]{geometry}
\def\BibTeX{{\rm B\kern-.05em{\sc i\kern-.025em b}\kern-.08em
    T\kern-.1667em\lower.7ex\hbox{E}\kern-.125emX}}

\usepackage{mdframed}
\definecolor{theoremcolor}{rgb}{0.97, 0.97, 0.97}
\definecolor{examplecolor}{rgb}{1, 1, 1.0}
\mdfsetup{
    innertopmargin=8pt,
    innerbottommargin=8pt,
    leftmargin=4pt,
    rightmargin=4pt,
    backgroundcolor=theoremcolor,
    linewidth=0pt,
}
\usepackage{xfrac}
\usepackage{comment}
\newmdtheoremenv[linewidth=0pt,innerleftmargin=4pt,innerrightmargin=4pt]{definition}{Definition}
\newmdtheoremenv[linewidth=0pt,innerleftmargin=4pt,innerrightmargin=4pt]{proposition}{Proposition}
\newmdtheoremenv[linewidth=0pt,innerleftmargin=0pt,innerrightmargin=0pt,backgroundcolor=examplecolor]{example}{Example}
\newmdtheoremenv{corollary}{Corollary}
\newmdtheoremenv{theorem}{Theorem}
\newmdtheoremenv{lemma}{Lemma}
\newmdtheoremenv{remark}{Remark}

\newcommand{\boldres}[1]{{\textbf{\textcolor{red}{#1}}}}

\newcommand{\ConMIL}{\textbf{\texttt{ConMIL}}}

\title{Smarter Together: Combining Large Language Models and Small Models for Physiological Signals Visual Inspection}

\author{
Huayu Li$^{1}$, Zhengxiao He$^{1}$, Xiwen Chen$^{2}$, Ci Zhang$^{3}$, Stuart F. Quan$^{4,5}$, William D.S. Killgore$^{6,7}$,\\ 
Shufen Wung$^{7,8}$, Chen X. Chen$^{8}$, Geng Yuan$^{3}$, Jin Lu$^{3}$, Ao Li$^{1,7}$\\[5pt]
\textit{$^{1}$Department of Electrical and Computer Engineering, University of Arizona, Tucson, AZ, USA}\\
\textit{$^{2}$School of Computing, Clemson University, Clemson, SC, USA}\\
\textit{$^{3}$School of Computing, University of Georgia, Athens, GA, USA}\\
\textit{$^{4}$Department of Medicine, University of Arizona, Tucson, AZ, USA}\\
\textit{$^{5}$Harvard Medical School and Brigham and Women’s Hospital, Boston, MA, USA}\\
\textit{$^{6}$Department of Psychiatry, University of Arizona, Tucson, AZ, USA}\\
\textit{$^{7}$BIO5 Institute, University of Arizona, Tucson, AZ, USA}\\
\textit{$^{8}$College of Nursing, University of Arizona, Tucson, AZ, USA}\\[5pt]
\textit{Corresponding author: Ao Li}\\
\textit{Department of Electrical and Computer Engineering, University of Arizona, Tucson, AZ, USA}\\
\textit{E-mail: aoli1@arizona.edu}\\
}

\date{}

\begin{document}
\maketitle

%\linenumbers

\newpage

\begin{abstract}
Large language models (LLMs) have shown promising capabilities in visually interpreting medical time-series data. However, their general-purpose design can limit domain-specific precision, and the proprietary nature of many models poses challenges for fine-tuning on specialized clinical datasets. Conversely, small specialized models (SSMs) offer strong performance on focused tasks but lack the broader reasoning needed for complex medical decision-making. \textcolor{black}{To address these complementary limitations, we introduce \ConMIL{} (Conformalized Multiple Instance Learning), a novel decision-support framework distinctively synergizes three key components: (1) a new Multiple Instance Learning (MIL) mechanism, QTrans-Pooling, designed for per-class interpretability in identifying clinically relevant physiological signal segments; (2) conformal prediction, integrated with MIL to generate calibrated, set-valued outputs with statistical reliability guarantees; and (3) a structured approach for these interpretable and uncertainty-quantified SSM outputs to enhance the visual inspection capabilities of LLMs.} Our experiments on arrhythmia detection and sleep stage classification demonstrate that \ConMIL{} can enhance the accuracy of LLMs such as ChatGPT4.0, Qwen2-VL-7B, \textcolor{black}{and MiMo-VL-7B-RL}. For example, \ConMIL{}-supported Qwen2-VL-7B \textcolor{black}{and MiMo-VL-7B-RL both achieves 94.92\% and 96.82\% precision on confident samples and (70.61\% and 78.02\%)/(78.10\% and 71.98\%) on uncertain samples for the two tasks}, compared to 46.13\% and 13.16\% using the LLM alone. These results suggest that integrating task-specific models with LLMs may offer a promising pathway toward more interpretable and trustworthy AI-driven clinical decision support.
\end{abstract}


\maketitle
\newpage
\section{Introduction}\label{sec:introduction}

Physiological signals, such as electrocardiograms (ECGs) and electroencephalograms (EEGs), offer a wealth of information about a patient’s health status through sequentially sampled physiological signals. Traditionally, clinicians have relied on visual inspection to interpret these signals, identifying trends and anomalies that guide diagnoses and treatment decisions. However, this manual process is time-intensive and prone to human error, especially in complex cases or under high workload conditions. Artificial Intelligence (AI) has emerged as a powerful tool to address these challenges, enabling automated interpretation of physiological signals. For clinicians, AI systems offer efficiency and accuracy by processing vast amounts of complex data and detecting critical patterns that may elude human observers. For patients, these systems can lead to earlier disease detection and more personalized care, particularly in underserved regions.

Despite these benefits, existing AI solutions face limitations. Small specialized models (SSMs)~\cite{shickel2017deep,che2018recurrent,ismail2020inceptiontime,wang2024medformer,chen2024timemil,early2024inherently} are task-focused and small-scale architectures designed for specific tasks such as arrhythmia detection or sleep stage classification. SSMs are the mainstream of the physiological signal domain which usually directly process the raw time series data. SSMs excel in their narrow domains but lack the broader reasoning capabilities required for complex clinical decision-making. 
On the other hand, large language models (LLMs) have rapidly gained traction in healthcare~\cite{nori2023can, singhal2023towards, nakari2024sleep}. The ability of LLMs to process visual-language tasks enable them to perform physiological signal interpretation by visual inspection~\cite{gunay2024accuracy, gunay2024comparison, zhu2024multimodal,sano2024exploration,zaboli2025exploring} with reasoning processes that mimic human clinicians. LLMs hold the promise of context-aware, intuitive decision-making, but their broad scope comes with drawbacks: limited domain-specific precision, high computational costs for fine-tuning, and restricted accessibility due to proprietary weights.

These contrasting strengths and weaknesses reveal a significant opportunity: \textit{How can we combine the task-specific expertise of SSMs with the contextual reasoning capabilities of LLMs to enable robust, interpretable, and reliable visual inspection for physiological signals?} To address this, we must reconsider the role of SSMs in the physiological signal domain. Rather than functioning as standalone predictors, SSMs can be reimagined as specialized, complementary modules that enhance the reasoning capabilities of LLMs. This reframing leads to a central question: \textit{How can we effectively integrate SSMs into the LLM workflow to maximize their utility while addressing their inherent limitations?} Achieving this synergy requires a thoughtfully designed framework that aligns with clinical workflows and ensures that the outputs of SSMs are both interpretable and actionable. By embedding the insights of SSMs into the LLM decision-making process, we can enable more informed, transparent, and collaborative clinical decisions, bridging the gap between task-specific precision and broad, context-aware reasoning.

A critical challenge in multimodal clinical AI is how to effectively use the predictions of SSMs to guide large language models (LLMs) in decision-making. Ideally, SSMs should reduce the likelihood of errors made by LLMs by providing higher prediction accuracy on specialized tasks. However, in practice, SSMs may not always surpass the accuracy of LLMs, even when trained specifically on physiological signal data. Inaccurate predictions from SSMs risk misleading LLMs, thereby compromising rather than enhancing decision-making accuracy. Furthermore, when SSM and LLM predictions conflict, it can be difficult for users to determine which model to trust. Beyond merely generating predictions, SSMs should offer deeper insights into the physiological signal data, enabling LLMs to validate and contextualize their outputs. This highlights the need for an integrated approach that ensures both interpretability and reliability, empowering LLMs to make more accurate and trustworthy clinical decisions. 

To address these limitations, we need approaches that offer both interpretability and uncertainty quantification. One promising strategy is the integration of Multiple Instance Learning (MIL) and Conformal Prediction. MIL is a weakly supervised learning technique that treats physiological signal data as collections of instances, such as individual time segments. It identifies which segments have most influence on a given classification or prediction~\cite{chen2024timemil,early2024inherently}. This detailed perspective not only enhances transparency but also allows LLMs to validate predictions by tracing them back to specific, clinically meaningful intervals. However, while MIL excels at interpretability, it does not inherently quantify uncertainty. Conformal prediction addresses this gap by producing set-valued predictions at predefined confidence levels~\cite{vovk2005algorithmic},  thus enabling uncertainty estimates with statistical guarantees under mild assumptions. We create a compelling paradigm that can bolster both transparency and trustworthiness. Ultimately, this integration can transform SSMs into safer and more actionable supportive plug-ins for visual inspection of LLMs on physiological signals: a crucial feature for human-in-the-loop clinical workflows where interpretability remains non-negotiable.

We introduce \ConMIL{} (\textbf{\texttt{Con}}formalized \textbf{\texttt{M}}ultiple \textbf{\texttt{I}}nstance \textbf{\texttt{L}}earning), a supportive SSM to enhance the visual inspection capabilities of LLMs on physiological signal data. \ConMIL{} offers two main advantages: First, unlike conventional MIL methods that only handle single positive class (e.g., single diagnosis), we proposed \textit{QTrans-Pooling} mechanism. This mechanism uses learnable class tokens and cross-attention to identify the most salient data segments for each class, facilitating intuitive understanding. This is particularly beneficial in scenarios with multiple diagnoses, outcomes, or treatment options. Second, \ConMIL{} incorporates conformal prediction to deliver class-specific confidence measures by dynamically adjusting prediction thresholds based on calibration sets. This approach ensures that every reported prediction is supported by a rigorously quantified level of reliability. To our knowledge, \ConMIL{} is the first to combine conformal prediction with MIL, marking a notable advancement in healthcare AI.

Crucially, \ConMIL{} embodies a strategic shift in how SSMs enhance clinical decision-making by leveraging LLMs for visual inspection of physiological signals. By providing interpretable, confidence-calibrated set-valued outputs, \ConMIL{} reduces the risk of misdiagnoses and misclassifications compared to opaque black-box models. The synergy between MIL’s interpretability and conformal prediction’s uncertainty quantification is designed to meet the demands of high-stakes healthcare contexts. Through comprehensive evaluations, including sleep stage and arrhythmia classification (see Section \ref{sec: Case studies}), we demonstrate that integrating \ConMIL{} with LLMs substantially improves diagnostic accuracy. These findings underscore the importance of combining interpretability and set-valued prediction to advance reliable and effective AI-informed clinical decision support.

\section{Related Works} \label{sec:related-works}

\subsection{Physiological Signal Classification}
Physiological signals, such as ECG and EEG signals, are critical for monitoring health, diagnosing diseases, and predicting clinical outcomes. Tasks like detecting arrhythmias from ECG \cite{ansari2023deep}, predicting seizures from EEG \cite{kuhlmann2018seizure}, and classifying sleep stages \cite{yue2024research} rely on accurate classification. While traditional approaches used feature engineering and classical machine learning \cite{dai2013atrial, faust2015wavelet}, deep learning models, including CNNs \cite{ismail2020inceptiontime,avanzato2020automatic,luo2024moderntcn,wang2024contrast}, and transformer-based architectures \cite{vaswani2017attention,eldele2021attention,li2024mts,wang2024medformer} now dominate the field due to their ability to handle complex temporal patterns.

\subsection{Time Series Multiple Instance Learning (MIL)}
MIL addresses tasks where labels exist at the bag level, with only some instances within a bag relevant to the label \cite{dietterich1997solving}. Applied to time series, MIL identifies critical segments that contribute to classification \cite{chen2024timemil,early2024inherently}, making it highly interpretable and effective in noisy medical data. Unlike post-hoc interpretation methods such as SHAP \cite{scott2017unified} or LIME \cite{ribeiro2016should}, MIL directly integrates interpretability into the modeling process, ensuring robust performance and precision in high-stakes domains like healthcare. \textcolor{black}{However, most existing time series MIL approaches,such as TimeMIL~\cite{chen2024timemil} and MILLET~\cite{early2024inherently}, focus on providing interpretability for a single final prediction or a global class token, which can limit their utility in scenarios involving multiple plausible diagnoses or latent states. Moreover, they often lack mechanisms for producing calibrated, set-valued predictions with formal uncertainty guarantees. \ConMIL{} directly addresses these limitations by introducing QTrans-Pooling, which enables fine-grained interpretability across multiple classes, and by seamlessly integrating conformal prediction into the MIL framework to support statistically rigorous uncertainty quantification.}

\subsection{Conformal Prediction}
Conformal prediction \cite{vovk2005algorithmic} offers set-valued predictions with predefined coverage, ensuring rigorous uncertainty quantification. Its adaptability to complex tasks like MIL remains underexplored but holds promise for enhancing reliability and interpretability in physiological signal classification. Compared to Bayesian approaches \cite{mackay1992practical} or Monte Carlo dropout \cite{gal2016dropout}, conformal prediction guarantees valid coverage is computationally efficient and non-parametric, making it particularly suited for medical applications where accuracy and interpretability are critical. \textcolor{black}{The recently proposed Monty Hall method~\cite{vishwakarma2024monty} claimed that providing LLMs with conformal prediction sets, narrowed options, can significantly improve their decision-making capabilities. However, relying solely on the prediction set is insufficient. To better support informed and reliable decisions, additional context such as per-class interpretabilityg is crucial. These richer inputs activate the strength of LLM in contextual reasoning and enable them to synthesize uncertainty-aware insights rather than simply choosing from a reduced label set.”}

\subsection{Clinical Decision-Making with LLMs}
LLMs are transforming clinical decision-making by supporting diagnostics, treatment recommendations, and risk assessments. MModels like Med-PaLM \cite{singhal2023towards} have demonstrated near-expert-level performance on standardized exams through domain-specific pretraining and structured prompts. They excel in nuanced medical reasoning, including step-by-step diagnostic tasks, and have shown potential in analyzing visual data—such as ECG \cite{gunay2024accuracy}, as well as in applications within sleep medicine \cite{nakari2024sleep}. Integration of LLMs into clinical workflows improves accuracy, reduces cognitive load, and expands the utility of AI in healthcare. 

\section{Conformalized Multiple Instance Learning}
This section delves into motivations, foundational principles, and algorithmic underpinnings of \ConMIL{}, providing a comprehensive explanation of its design and technical validity.

\subsection{Problem Formulation}
\label{sec:problem-formulation}
We first formulate the problem of classification of physiological signals. Given an input $X=[x_1,x_2,\dots,x_T]$, where each $x_t\in\mathbb{R}^c$ represents a $c$-variates feature vector recorded at time $t$, the goal is to assign the correct class label $y\in\mathcal{Y}= \left\{1,\dots,K \right\}$ from $K$ possible classes. These classes may represent different medical conditions or diagnostic categories, such as arrhythmia types or sleep stages. 

Physiological signal classification can naturally be formulated as a MIL problem by treating signals as ``bags'', where each time step is considered an instance. This can be formally defined as:
\begin{align}
     y_i=\begin{cases}
 &0,\quad \text{iff} \sum^{T}_{t=1}y_i^t=0, y_i^t \in \left\{0, 1\right\} \\
 &1,\quad \text{otherwise},
\end{cases}
\end{align}
where $y^t_i$ represents the instance-level label, indicating whether an event of interest has occurred at that time step $t$. For a multiclass classification problem with $K$ classes, Physiological signal classification can be extended by performing several \emph{one-vs-rest} binary MIL tasks without violating the assumptions of MIL. Formally, this is defined as:
\begin{align}
     y_{i,k}=\begin{cases}
 &0,\quad \text{iff} \sum^{T}_{t=1}y_{i,k}^t=0, y_{i,k}^t \in \left\{0, 1\right\} \\
 &1,\quad \text{otherwise},
\end{cases}
\end{align}
where $y^t_{i,k}$ denotes a time point $t$ with significant contribution to class $k\in\left\{1,\dots ,K\right\}$. The final bag-level label is determined by majority voting, indicated by $y_i=\arg\max_k\sum^{T}_{t=1}y_{i,k}$.

Conformal prediction is a framework designed to provide reliable uncertainty estimates for model predictions. Instead of outputing a single predicted label, the conformal prediction generates a set value prediction that contains the true label with a user-defined probability, known as \emph{coverage}. In medical settings, for example, conformal prediction can offer a set of possible diagnoses, ensuring that the correct diagnosis is included in this set with at least $1-\alpha$ confidence, where $\alpha$ represents the error rate or miscoverage rate. Formally, given the ground truth label $y$, the input time series, our goal is to construct a set-valued prediction $\mathcal{S}_{\alpha}(X)\subseteq \mathcal{Y}$ such that:
\begin{align}
    \mathbb{P}(y\in\mathcal{S}_{\alpha}(X)) \geq 1-\alpha,
\end{align}
for a pre-specified miscoverage rate, such as 10\%.

\begin{figure}
\centering
\includegraphics[width=0.99\columnwidth]{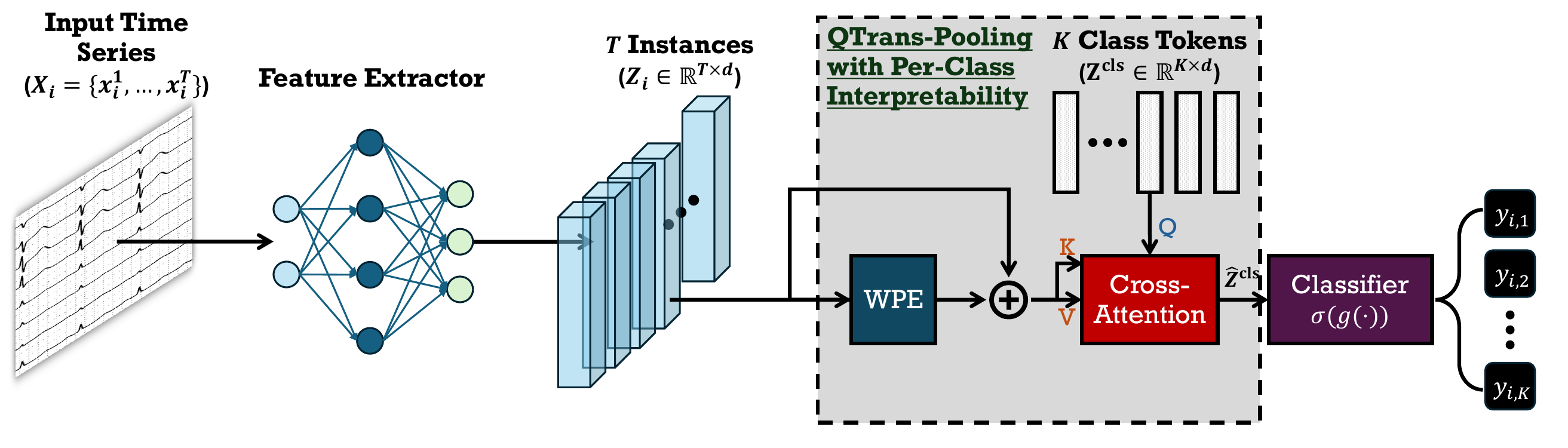}
\caption{Illustration of the MIL model used in \ConMIL{}. Our model is adapted from the one proposed by \cite{chen2024timemil} which uses InceptionTime \cite{ismail2020inceptiontime} as the feature extractor and encodes temporal correlations with Wavelet Positional Encoding (WPE)~\cite{chen2024timemil}. Instead of \textit{Trans-Pooling}, we have introduced the \textit{QTrans-Pooling} for enhanced per-class interpretability.} 
\label{fig:milnet}
\end{figure}

\subsection{\textit{QTrans-Pooling} with per-class interpretability}\label{sec:Qtrans}

We first recap the approximation of the general MIL methods stated in \cite{ilse2018attention,shao2021transmil},
\begin{theorem}\label{thm:1}
Let $S$ be a symmetric score function that is $(\delta_{\varepsilon}, \varepsilon)$-continuous with respect to the Hausdorff distance $d_H(\cdot, \cdot)$, meaning that for all $\varepsilon > 0 $ and $i\neq j$, if $d_H(X_i, X_j) < \delta_{\varepsilon} $, then
\begin{align}
|S(X_i) - S(X_j)| < \varepsilon.
\end{align}
Then, for any invertible map $\psi: \mathcal{X} \rightarrow \mathbb{R}^d $, there exist continuous functions $g $ and $\phi$ such that
\begin{align}
\left| S(X_i) - g\left(\psi\left( \{ \phi(x_i^t) : x_i^t \in X_i \} \right)\right) \right| < \varepsilon,
\end{align}
where $x_i^t$ denotes the $t$-th instance in the $i$-th bag $X_i$.
\end{theorem}

This theorem presents the pipeline of MIL paradigm that existing architectures typically consist of three main components: 1) a feature extractor $\phi(\cdot)$, which processes the input $X$ and extracts a set of $d$-dimensional feature embeddings $Z\in\mathbb{R}^{T\times d}=[z_1,z_2,\dots,z_T]$; 2) a MIL pooling layer $\psi(\cdot)$, which computes a weighted feature vector $Z^{pool}\in\mathbb{R}^{d}$ along with corresponding attention weights $A=[a_1,a_2,\dots,a_T]$, where $a_t\in\left\{0, 1\right\}$, for each instance, offering interpretability by highlighting the contributions of individual time steps to the final decision; and 3) a classifier $g(\cdot)$ which takes the pooled feature vector $Z^{pool}$ as input and outputs class probabilities $\hat{y}$. Formally, the MIL can be formulated as the following general process:
\begin{align}
\label{eq: mil-general}
    Z=\phi(X),\left\{Z^{pool},A\right\} = \psi(Z), \hat{y} = g(Z^{pool}).
\end{align}

Our study begins by examining the limitations of two recent MIL methods—Conjunctive Pooling \cite{early2024inherently} and \textit{Trans-Pooling} \cite{chen2024timemil}—within the context of time series analysis.
Conjunctive pooling \cite{early2024inherently} is a novel pooling method that independently applies instance-wise attention and classification to the time point embeddings, followed by scaling the time point predictions using the corresponding attention values. Slightly different from the general process, conjunctive pooling is defined as:
\begin{align}
\hat{y}_t = g(z_t), \hat{y}=\frac{1}{T}\sum_{t=0}^T (a_t\hat{y}_t), a_t \in\left\{0, 1\right\}=\sigma(W^{A}z_t).
\end{align}
In this formulation, $a_t$ is the attention value assigned to each time point $t$, calculated by applying a sigmoid activation $\sigma$ to the feature vectors $W^{A}z_t$. 

On the other hand, \textit{Trans-Pooling} \cite{chen2024timemil} introduces the use of Transformers~\cite{vaswani2017attention} in MIL, leveraging their self-attention mechanism to capture dependencies between a learnable class token $z^{\text{cls}}$ and each instance. Formally, given the concatenation $Z^{\text{cls}}$ of class token $z^{\text{cls}}$ and instance embeddings $Z$, the process is defined as:
\begin{align}
\hat{y}=g(\hat{z}^{\text{cls}}),\hat{Z}=\text{Attention}(W^{Q}Z^{\text{cls}}, W^{K}Z^{\text{cls}}, W^{V}Z^{\text{cls}})W^{O}=[\hat{z}^{\text{cls}}, \hat{z}_1, \dots, \hat{z}_T],
\end{align}
where $W^{Q}$, $W^{K}$, $W^{V}$,and $W^{O}$ are trainable parameters, the self-attention mechanism is computed as:
\begin{align}
\label{eq:attention}
    \text{Attention}(Q,K,V)=\softmax(\frac{QK^\dagger}{\sqrt{d}})V,
\end{align}
and $\hat{Z} = [\hat{z}^{\text{cls}}, \hat{z}_1, \dots, \hat{z}_T]$ is the concatenation of the class token $\hat{z}^{\text{cls}}$ and the instance embeddings after self-attention $\hat{z}_1, \dots, \hat{z}_T$. The attention mechanism enables the model to weigh the importance of each instance in relation to the class token, refining the pooled representation $\hat{z}^{\text{cls}}$ for final classification.

We observed that while both Conjunctive Pooling and Trans-Pooling offer a degree of interpretability, they fall short of delivering per-class interpretability in multi-class classification settings. Conjunctive Pooling aggregates instance-level predictions using a single attention weight, making it difficult to disentangle the contribution of individual time steps to each potential class. Trans-Pooling incorporates attention but relies on a single global class token, which inherently constrains interpretation to the dominant or final prediction. As a result, both methods may obscure valuable insights in cases where multiple diagnostic outcomes are plausible, an important limitation in applications involving comorbidities or ambiguous physiological patterns.

To address this, we proposed the \textit{QTrans-Pooling} used in \ConMIL{} as illustrated in Figure~\ref{fig:milnet}. Inspired by the work of \cite{chen2024timemil}, we also implement \textit{QTrans-Pooling} with learnable class tokens and Transformers. Considering the instance importance can be measured by the attention maps between class tokens and instances, to achieve per-class interpretability can be implemented via assigning each class an independent class token. To this end, we introduce a separate class token $z_k^{\text{cls}}$, allowing the model to compute class-specific attention weights for each instance. In \textit{QTrans-Pooling}, cross-attention between the class tokens and instance embeddings used to achieve per-class interpretability. Given the class tokens $Z^{\text{cls}} = [z_1^{\text{cls}}, \dots, z_K^{\text{cls}}]$, and instance embeddings $Z = [z_1, \dots, z_T]$, the cross-attention is computed as follows:
\begin{align}
    \text{Attention}(Q,K,V)=\sigma(\frac{QK^\dagger}{\sqrt{d}}+b)V,\quad b = -\log(T),
\end{align}
with $Q=W^{Q}Z^{\text{cls}}$, $K=W^{K}Z$, $V=W^{V}Z$. The bias term is defined as $b$ and $\sigma$ is the Sigmoid function. Sigmoid attention~\cite{ramapuram2024theory} is used since the softmax function may possibly cause attention to focus on a few features while ignoring other information. This process allows the class token $z_k^{\text{cls}}$ to attend to the embeddings of the instance, resulting in an attention map $A_k=[a_{k,1}, a_{k,2}, \dots, a_{k,T}]$, where each $a_{k,t}$ represents the importance of time step $t$ for class $k$. 

The use of cross-attention between class tokens and instance embeddings is pivotal for achieving per-class interpretability. By assigning a distinct learnable class token to each class, the model can independently assess the relevance of each instance to each class, rather than relying on a shared representation. This architectural design enables disentangled attention pathways, allowing the model to highlight which time segments are most informative for each potential outcome. Such capability is particularly important in multi-class physiological signal classification, where patients may exhibit patterns indicative of multiple comorbid conditions, and clinical decisions require transparent, class-specific explanations.

\textcolor{black}{To illustrate this concept in physiological signals, consider the task of arrhythmia detection from a 10-second ECG recording (the 'bag'). The overall recording might be labeled as containing an 'arrhythmia'. Multiple Instance Learning aims to identify which specific, shorter segments such as the R-R interval(the 'instances') within that 10-second window are most indicative of the arrhythmia, even if other parts of the ECG appear normal. Our proposed QTrans-Pooling mechanism further refines this by providing per-class interpretability; for example, it could highlight specific instances that point towards 'Atrial Fibrillation' while simultaneously showing different instances that might suggest 'Ventricular Tachycardia', if the model is considering multiple arrhythmia types.}

% Moreover, the use of independent class tokens allows the model to disentangle overlapping features across different classes, resulting in more precise and interpretable predictions. For example, in ECG analysis, specific time segments may correspond to different heart rhythm abnormalities, and per-class tokens help clarify which signal features are relevant to each diagnosis. We further provide an information-theoretical justification to validate this property.

\begin{theorem}\label{thm:2}
Compared with \textit{Trans-Pooling} \cite{chen2024timemil}, applying \textit{\textit{QTrans-Pooling}} with Per-Class Interpretability reduces class-wise variability in the latent space. This reduction simplifies the task of learning effective classification boundaries. 

Specifically, we use class-conditioned entropy $H$ to quantify the homogeneity of the latent feature space within each class. The following inequality holds:
\begin{align}
         H(\hat{Z}|\hat{z}^{cls}_1,\cdots,\hat{z}^{cls}_K)\leq H(\hat{Z}|\hat{z}^{cls}),
\end{align}  
where $H(\hat{Z}|\hat{z}^{cls}_1,\cdots,\hat{z}^{cls}_K)$ represents the class-conditioned entropy when using \textit{QTrans-Pooling} with distinct class tokens for each class, and $H(\hat{Z}|\hat{z}^{cls})$ represents the class-conditioned entropy under \textit{Trans-Pooling}, which relies on a single global class token.
A lower class-conditioned entropy implies tighter clustering of instances within each class, which in turn leads to simpler and more robust decision boundaries in the feature space. \end{theorem}

\begin{proof}
    The proof is straightforward, as conditioning on additional information can only maintain or lower the conditional information \cite{cover1999elements}. 
\end{proof}

\begin{remark}
  A reduction in class-conditioned entropy implies that features within the same class form tighter clusters, thereby reducing variability between classes. Consequently, minimizing class-wise variability across all classes facilitates the identification of simpler and more robust decision boundaries for classification.
\end{remark}

\begin{remark}
Lower class-conditioned entropy enhances confidence in assigning instances to their correct classes, improving classification reliability and interpretability.
\end{remark}

We also compare the computational complexity between \emph{QTrans-Pooling} and \emph{Trans-Pooling} from their attention mechanisms. Trans-Pooling relies on self-attention, which scales with $\mathcal{O}(T^{2}d)$ where $T$ is the number of instances and $d$ is the feature dimension. In contrast, QTrans-Pooling employs cross-attention between instance embeddings and multiple class tokens, reducing the computational cost to $\mathcal{O}(TKd)$, where $K$ is the number of classes and typically much smaller than $T$, This results in QTrans-Pooling being more efficient in scenarios with a large number of instances, as it avoids the quadratic complexity of self-attention while still maintaining interpretability.

\subsection{Conformalizing the MIL model} \label{sec:crc}
We now discuss how to achieve reliable set-valued prediction with per-class coverage guarantees by integrating conformal prediction into \ConMIL{}. \textcolor{black}{For example, in the context of sleep stage classification from an EEG segment, if a clinician sets the desired confidence level to 95\% (i.e., an error rate $\alpha$=0.05), \ConMIL{}, through conformal prediction, would not just output a single sleep stage prediction. Instead, it might generate a prediction set, such as \{N2, REM\}. This set is accompanied by the statistical guarantee that the patient's true current sleep stage is contained within this set \{'N2', 'REM'\} with at least 95\% probability, provided the assumptions of the conformal prediction framework are met (e.g., data exchangeability or adjustments for distribution shift as discussed in Theorems \ref{theorem:CRC_Exchangeable} and \ref{theorem:CRC_NONExchangeable}). This provides a transparent and statistically grounded way for clinicians to understand the model's uncertainty.}

To this end, we aim to construct set-valued prediction that controls the False Negative Rate (FNR) for each class $k$, ensuring that the true class is included in the set-valued prediction with high probability. We consider implementing \ConMIL{} under the split conformal prediction setting~\cite{papadopoulos2002inductive}. We reserve a calibration dataset which remains unseen during training the MIL model; we seek to construct a set-valued prediction $\mathcal{S}_{\alpha}(X_\text{test})$ of the new test data that is valid in the following sense:
\begin{align}
\label{eq:fnr_1}
\nonumber
    \mathbb{P}(y_\text{test}\in\mathcal{S}_{\alpha}(X_\text{test})|y_\text{test}=k) \geq 1-\alpha,\forall k\in\left\{1,\dots ,K\right\},
\end{align}
where $\alpha$ is the predefined confidence level and $\mathcal{S}_{\alpha}(X)$ represents the set-valued prediction for input $X$. This guarantees that for each class $k$, the true class is captured with at least $1-\alpha$ probability,  controlling the likelihood of missing a true positive (i.e., controlling the FNR).

For each class $k$, the FNR in a dataset $\mathcal{D}=\left\{(X_1,y_1),\dots,(X_N,y_N)\right\}$, where $N$ represents the number of samples, can be defined in terms of a risk function. The risk function quantifies the empirical likelihood that the presence of class $k$ is not correctly predicted in the dataset. Formally, the risk for class $k$ is defined as:
\begin{align}
    \mathcal{R}_k(\mathcal{D})=1-\frac{1}{N}\sum^{N_k}_{i=0}\mathbb{I}(y_i\in\mathcal{S}_{\alpha}(X_i)),
\end{align}
where $N_k$ is the number of samples in $\mathcal{D}$ with true class label $k$, $\mathbb{I}$ is the indicator function that equals 1 if the condition is satisfied and 0 otherwise, $y_i\in\mathcal{S}_{\alpha}(X_i)$ indicates whether the true label is included in the set-valued prediction. In binary classification, the model's final prediction is often decided by thresholding the predicted probability with a predefined threshold $\lambda$. The binary prediction $\hat{y}_i$ is determined by rounding the predicted probability $\hat{p}(X_i)$:
\begin{align}
    \hat{y}_i=\begin{cases}
 &1,\quad \text{if}\quad \hat{p}(X_i)\geq 1-\lambda, \\
 &0,\quad \text{otherwise}.
\end{cases}
\end{align}
The per-class FNR is controlled by selecting an appropriate threshold $\lambda$ such that the set-valued prediction includes the true class label with the desired coverage probability. To generalize this for multi-class classification, we assign a threshold $\lambda_k$ for each class $k$, forming a set of thresholds  $\lambda=\left\{\lambda_1,\dots, \lambda_K\right\}$. These thresholds act as additional parameters for constructing the set-valued prediction $\mathcal{S}_{\alpha}(X_i, \lambda)$, which is defined as:
\begin{align}
    \mathcal{S}_{\alpha}(X_i, \lambda) := \left\{\hat{y}_{i,k}:\hat{p}(X_i)\geq 1-\lambda_k\right\}.
\end{align}
This formulation ensures that the set-valued prediction includes all classes for which the predicted probability exceeds the class-specific threshold $\lambda_k$.

For picking the threshold $\lambda$, Conformal Risk Control (CRC) \cite{angelopoulos2022conformal} is a natural choice, as it provides a principled framework for balancing coverage guarantees with risk control. CRC extends traditional conformal prediction by incorporating the notion of risk minimization, allowing us to adjust thresholds dynamically to ensure that the set-valued prediction achieves a predefined level of reliability while controlling FNRs across multiple classes.

\begin{definition}
(\textbf{Conformal Risk Control} \cite{angelopoulos2022conformal})
Given an arbitrary bounded risk function $\mathcal{R}_{\lambda}(\cdot)\in(-\inf,B]$ for some $B<\inf$ that is monotonically non-increasing with respect to the threshold $\lambda$, the goal of CRC is to select the smallest possible threshold such that the risk remains controlled at a predefined level. Formally, for a given calibration dataset $\mathcal{D}_{\text{cal}}$ and a risk level $\alpha$, CRC selects $\hat{\lambda}$ such that the risk on the test data $\mathcal{D}_{\text{test}}$ is controlled:
\begin{align}
    \mathbb{E}[\mathcal{R}_{\hat{\lambda}}(\mathcal{D}_{\text{test}})] \leq \alpha,
\end{align}
where the optimal threshold $\hat{\lambda}$ is selected by solving the following optimization problem:
\begin{align}
\label{eq:optimal-threshold}
    \hat{\lambda} = \inf \left\{\lambda: \frac{n}{n+1}\mathcal{R}_{\hat{\lambda}}(\mathcal{D}_{\text{cal}})+\frac{B}{n+1}\leq \alpha \right\}
\end{align}
where $\mathcal{R}_{\hat{\lambda}}(\mathcal{D}_{\text{cal}})$ is the empirical risk based on the calibration data.%, and $B$ is the upper bound on the risk function.
\end{definition}

By using CRC,  we are able to dynamically adjust the threshold $\lambda$ on the calibration dataset to achieve the per-class coverage. To demonstrate that CRC provides guarantees, we first relax the assumption of independent and identically distributed (i.i.d.) data and instead consider an exchangeable data distribution.
\begin{definition}
(\textbf{Exchangeable data distribution})
    Let $\mathcal{X}$ and $\mathcal{Y}$ represent the time series input and label space, respectively. A data distribution in $\mathcal{X}\times \mathcal{Y}$ is said to be exchangeable if and only if the following condition holds:
    \begin{align}
        \mathbb{P}((X_{\pi(1)},y_{\pi(1)}), \dots,(X_{\pi(N)},y_{\pi(N)}))=\mathbb{P}((X_{1},y_{1}), \dots,(X_{N},y_{N}))
    \end{align}
    for any finite sample $\left\{(X_{n},y_{n})\right\}^N_{n=1}$, where $\pi(\cdot)$ denotes an arbitrary permutation of the dataset indices.
\end{definition}

Under the assumption of an exchangeable data distribution, CRC provides exact guarantees as stated in the following theorem.
\begin{theorem}
\label{theorem:CRC_Exchangeable}
(\textbf{Conformal risk control over exchangeable data distribution})
Given a risk function $\mathcal{R}_{\lambda}$ that is right-continuous and non-increasing with respect to $\lambda$, and a calibration dataset $\mathcal{D}_{cal}=\left\{\mathcal{D}_n\right\}_{n=1}^{N}$ containing $N$ data points, we also denote any test point as $\mathcal{D}_{N+1}$. If the calibration and test datasets are sampled from an exchangeable distribution $\mathcal{X}\times \mathcal{Y}$, then the following conditions hold: 
\begin{align}
    \mathcal{R}_{\lambda_{\text{max}}}(\mathcal{D}_n)\leq \alpha, \sup_\lambda\mathcal{R}_{\lambda}(\mathcal{D}_n)\leq B<\infty \quad almost\quad surely.
\end{align}
and thus, 
\begin{align}
    \mathbb{E}[\mathcal{R}_{\hat{\lambda}}(\mathcal{D}_{N+1})] \leq \alpha.
\end{align}
where $\hat{\lambda}$ is the optimal threshold selected by CRC as per Eq. \ref{eq:optimal-threshold}.
\end{theorem}

\begin{proof}
The proof is provided in \cite{angelopoulos2022conformal}.
\end{proof}

In real-world applications, physiological signal data are often sampled from non-stationary distributions, meaning the datasets may not adhere to the assumption of exchangeable data distribution. For instance, data might be collected from different individuals, across diverse demographic groups, or using various hardware conditions. In such cases, it becomes essential to assess the conformal guarantees under distributional shifts.

To address this, we rely on the results of non-exchangeable split conformal prediction as introduced by \cite{barber2023conformal}, which allows us to quantify the coverage gap $\Delta_{\text{Cov}}=\alpha-\alpha^{*}$, where $\alpha$ is the specified risk level and $\alpha^{*}$ is the observed risk level. This gap can be bounded using the total variation (TV) distance. 
\begin{lemma}
(\textbf{Total variation (TV) bound})
    Given a bounded function $h:\Omega\rightarrow [0,B]$ on a measurable space $(\Omega,\mathcal{A})$ and let $P$ and $Q$ be two probability measures on $(\Omega,\mathcal{A})$, then:
    \begin{align}
        %|\mathbb{E}_{P}[h]-\mathbb{E}_{Q}[h]|\leq\sup_{h}|\mathbb{E}_{P}[h]-\mathbb{E}_{Q}[h]|=d_{\text{TV}}(P,Q).
        |\mathbb{E}_{P}[h]-\mathbb{E}_{Q}[h]|\leq Bd_{\text{TV}}(P,Q)
    \end{align}
    where $d_{\text{TV}}(P,Q)=\frac{1}{2}\sup_{h}|\mathbb{E}_{P}[h]-\mathbb{E}_{Q}[h]|$.
\end{lemma}
Intuitively, $d_{\text{TV}}$ measures the largest distance between two distributions. Using this measure, we can bound the coverage guarantee of CRC when the test data comes from a non-exchangeable distribution. In this scenario, we assume that while the calibration set is sampled from an exchangeable distribution, the test samples are drawn from a distinct distribution.

\begin{theorem}
\label{theorem:CRC_NONExchangeable}
(\textbf{Conformal risk control over non-exchangeable data distribution})
    Given risk function $\mathcal{R}_{\lambda}$ with the same properties as Theorem \ref{theorem:CRC_Exchangeable}, assume that the calibration data $\mathcal{D}_{cal}=\left\{\mathcal{D}_n\right\}_{n=1}^{N}$ and the test data point $\mathcal{D}_{N+1}$ are sampled from non-exchangeable distributions. Under these conditions, CRC provides the following coverage guarantee:
    \begin{align}
        \mathbb{E}[\mathcal{R}_{\hat{\lambda}}(\mathcal{D}_{N+1})] \leq \alpha+B\sum_{n=1}^{N}d_{\text{TV}}(\mathcal{D}_n,\mathcal{D}_{N+1}).
    \end{align}
    where the coverage gap is bounded by the total variation distance:
    \begin{align}
        \Delta_{\text{Cov}}=B\sum_{n=1}^{N}d_{\text{TV}}(\mathcal{D}_n,\mathcal{D}_{N+1})
    \end{align}
\end{theorem}

\begin{proof}
Proof refers to \cite{barber2023conformal}.
\end{proof}
Thus, in the non-exchangeable setting, the coverage guarantee provided by CRC remains valid, albeit with an additional term $\Delta_{\text{Cov}}=B\sum_{n=1}^{N}d_{\text{TV}}(\mathcal{D}_n,\mathcal{D}_{N+1})$, which quantifies the deviation of the test distribution from the calibration distribution. This term ensures that CRC accounts for the shift between the two distributions, providing a more realistic coverage estimate under non-stationary conditions.

\begin{figure*}
\centering
\includegraphics[width=0.99\textwidth]{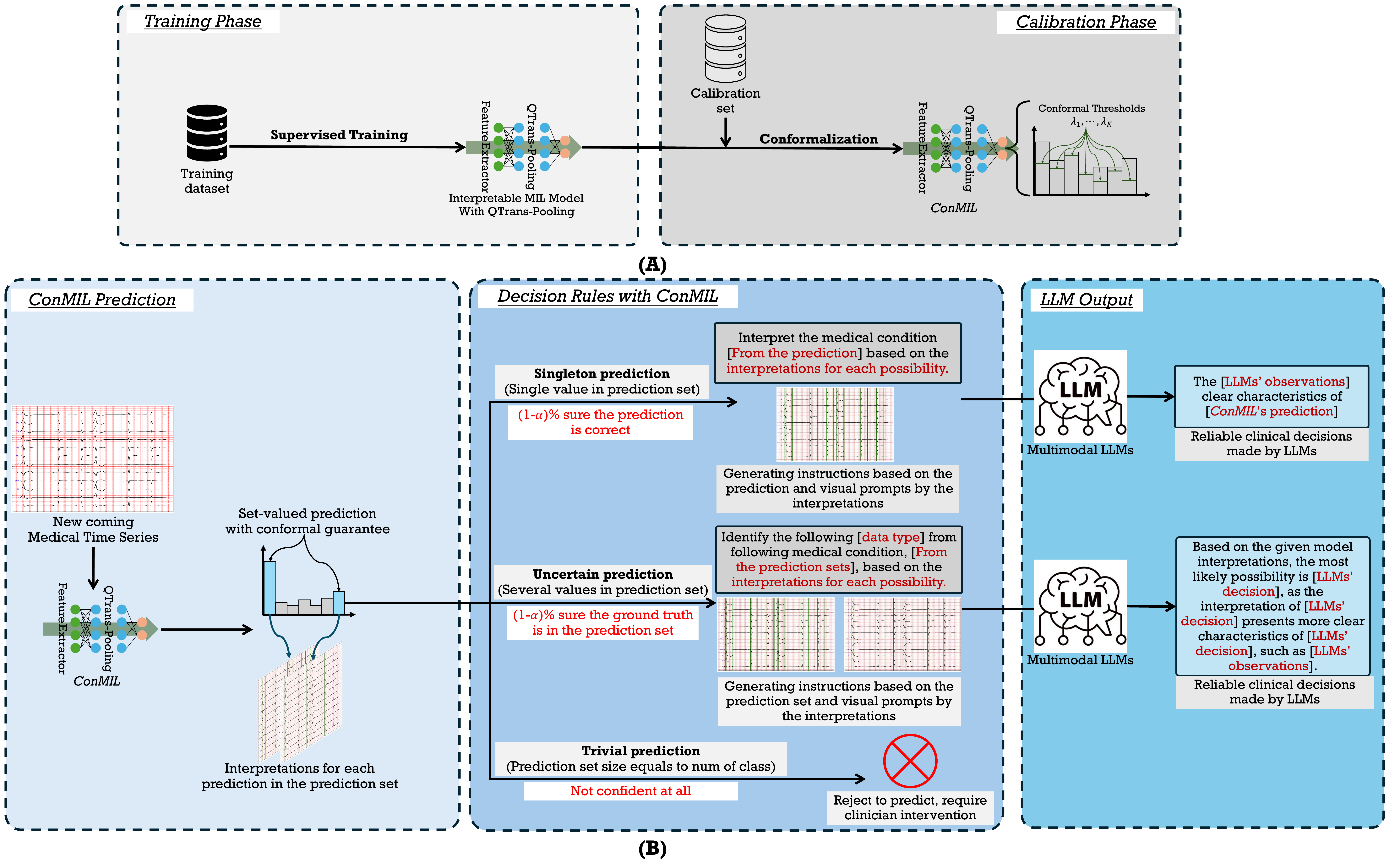}
\caption{Overview of the \ConMIL{} workflow for enhancing LLM-driven visual inspection of physiological signals. (A) Training and Calibration Process: \ConMIL{} undergoes supervised training with \textit{QTrans-Pooling} for per-class interpretability and calibration with conformal prediction to generate reliable thresholds $(\lambda_1,\dots,\lambda_K)$. This ensures set-valued predictions with guaranteed confidence levels. (B) Decision Support Framework: During deployment, \ConMIL{} generates set-valued predictions and interpretive visual prompts for new physiological signals. Based on the prediction type (singleton, uncertain, or trivial), the LLMs interpret the data to make reliable clinical decisions or flag ambiguous cases for clinician intervention, ensuring accuracy, transparency, and safety in healthcare workflows.}
\label{fig:ConMIL-Workflow}
\end{figure*}

\section{Enhancing the Visual Inspection Capabilities of LLMs with \ConMIL{}}
This section explores how \ConMIL{} enhances the ability of LLMs to perform visual inspection of physiological signals, detailing its integration and functional contributions within the workflow.

\subsection{\ConMIL{} as supportive plug-in: Pre-and-post deployment}
As previously noted, finetuning LLMs often encounters significant barriers. To address this, \ConMIL{} as a supportive SSM, serves as the primary domain-specific model. It augments pretrained LLMs, enabling robust and interpretable visual inspection of physiological signals. The integration of \ConMIL{} with LLMs follows a two-phase pipeline, as illustrated in Figure~\ref{fig:ConMIL-Workflow} (A).

During training, \ConMIL{} utilizes a MIL model enhanced with the \textit{QTrans-Pooling} mechanism (see Section~\ref{sec:Qtrans}) to extract meaningful patterns from physiological signals. This design enables the model to capture clinically relevant features while maintaining per-class interpretability. Such interpretability is essential in multi-class medical tasks, such as distinguishing between sleep stages or identifying different types of arrhythmia.

In the calibration phase, the supervised MIL model is conformalized using CRC techniques (Section~\ref{sec:crc}) to calibrate the MIL model on the calibration set, resulting in reliable conformal thresholds. These thresholds enable \ConMIL{} to generate set-valued predictions with predefined confidence guarantees, accompanied by rigorous uncertainty quantification. This calibration step enhances the reliability of clinical decision-making.

Post-deployment, \ConMIL{} generates set-valued predictions for new physiological signal data, supplemented by interpretive visual and textual prompts. These outputs are processed by LLMs, which apply their broader contextual reasoning to refine diagnostic conclusions. For instance, \ConMIL{} can highlight key time-series segments associated with specific diagnoses on visual plots while providing confidence-calibrated predictions. LLMs interpret these predictions and visual explainations to generate actionable insights, ensuring they are guided by domain-specific expertises while maintaining versatility in clinical workflows.

\subsection{Decision Rules with \ConMIL{}}\label{sec:decision-rules}
The decision rules governing the integration of \ConMIL{} and LLMs, as shown in Figure~\ref{fig:ConMIL-Workflow} (B), are designed to enhance reliability and trustworthiness in clinical settings. These rules leverage the set-valued predictions from \ConMIL{} and their interpretation by LLMs.

\paragraph{Singleton Prediction (High Confidence)}
When \ConMIL{} generates a singleton prediction (a single class in the prediction set), LLMs interpret the result directly using the associated interpretability prompts. The $(1-\alpha)\%$ confidence guarantee ensures high reliability, enabling the generation of precise clinical recommendations with minimal uncertainty.

\paragraph{Uncertain Prediction (Moderate Confidence)}
For predictions containing multiple possible classes, LLMs rely on the per-class interpretability provided by \ConMIL{} to prioritize the most likely diagnosis. The decision is guided by visual and textual explanations for each potential class, fostering a transparent and collaborative diagnostic process. For example, the LLMs can explain why one diagnosis is more plausible than others based on key features highlighted by \ConMIL{}.

\paragraph{Trivial Prediction (Low Confidence)}
If \ConMIL{} produces a prediction set encompassing all possible classes, signaling low confidence, the case is flagged for clinician review. This ensures that ambiguous or high-risk scenarios are escalated to human experts, prioritizing patient safety and preventing potential misdiagnoses. In addition to abstaining from uncertain predictions, our framework also provides fine-grained visual explanations by highlighting the most influential temporal segments within the time series. This interpretability is enabled by the MIL backbone, which identifies salient intervals that contribute to each prediction. These highlighted regions offer clinicians an intuitive view into the model’s reasoning process, facilitating more informed and efficient decision-making during clinician review.

\section{Experimental Results}

\subsection{Study design}
Our study is structured in four key parts:
\begin{enumerate}
    \item \textbf{Comparative Analysis of LLMs with and without \ConMIL{} Support (Section \ref{sec: trust_but_verify}):} We compare the diagnostic accuracy and reliability of LLMs operating independently versus those supported by \ConMIL{} under the proposed decision framework. This analysis quantifies the extent to which \ConMIL{} enhances LLM-based clinical decision-making.
    \item \textbf{Case Studies (Section \ref{sec: Case studies}):}  We conduct in-depth case studies involving samples where the top-1 prediction from conventional SSMs is incorrect. These cases illustrate how \ConMIL{}'s set-valued predictions and per-class interpretability assist LLMs in arriving at more accurate and transparent decisions.
    \item \textbf{“Interviewing the LLMs” (Section \ref{appendix:Going-deeper}):} To further explore the influence of \ConMIL{} on LLM decision-making, we design an “interview” protocol. This controlled experiment evaluates LLM responses under varying levels of model support—from no assistance to full access to \ConMIL{}'s interpretability and uncertainty quantification—providing qualitative insight into the model’s reasoning process.
\end{enumerate}

\subsection{Datasets and experimental setups}
We conduct our experiments only using the publicly available, de-identified datasets(SleepEDF~\cite{kemp2000analysis} and PTB-XL~\cite{wagner2020ptb}) for this study, ensuring compliance with ethical standards. These datasets are used for sleep stage classification and arrhythmia classification, tasks that are typically performed through visual inspection by clinicians or specialists in clinical settings.

\textbf{SleepEDF dataset~\cite{kemp2000analysis}}:
\textcolor{black}{The SleepEDF database, specifically its expanded version (sleep-edfx), contains 197 whole-night PolySomnoGraphic (PSG) sleep recordings. For this study, we utilized the "Sleep Cassette (SC)" portion, comprising 153 recordings from a 1987-1991 study focused on age-related sleep effects in 78 healthy Caucasian individuals aged 25-101 years. Participants were not administered sleep-related medication, ensuring the data reflect natural sleep patterns. These home-recorded PSGs typically spanned about 20 hours each over subsequent day-night periods and included EEG (from Fpz-Cz and Pz-Oz locations, with Fpz-Cz used in this study), EOG (horizontal), and chin EMG signals, originally sampled at 100 Hz. Sleep stages were manually annotated by trained technicians in 30-second epochs according to the Rechtschaffen and Kales (R\&K) standard~\cite{RKManual}. The stages include Wake (W), REM (R), and Non-REM stages (N1, N2, N3, where N3 and N4 from R\&K are often combined, and N1, N2, N3 are used in newer American Academy of Sleep Medicine standards~\cite{berry2012aasm}). For this study, focusing on the Fpz-Cz EEG channel, we segmented recordings into 3000-point windows. This processing yielded a total of 42,308 samples. Based on analyses of a comparable 20-subject subset of SleepEDF, the approximate distribution of these epochs is: \{Wake 19.6\%, N1 6.6\%, N2 42.1\%, N3 13.5\%, and REM 18.2\%\}. A subject-wise split is implemented, with 60\% of subjects allocated to the training set, 20\% to the validation set, and 20\% to the test set.}

\textbf{PTB-XL dataset~\cite{wagner2020ptb}}:
\textcolor{black}{The PTB-XL dataset is a large, comprehensive, and publicly available clinical electrocardiography resource, containing 21,799 10-second 12-lead ECGs from 18,869 patients (subjects). The patient cohort is balanced by sex (52\% male, 48\% female) and covers a wide age range (0-95 years, median 62 years). Recordings were collected with Schiller AG devices between October 1989 and June 1996. The ECGs are annotated by up to two cardiologists, assigning potentially multiple statements from 71 SCP-ECG standard diagnostic, form, and rhythm statements. For diagnostic purposes, these are hierarchically organized into five coarse superclasses. The distribution of these superclasses across the 21,799 records is: Normal ECG (NORM): 9,514; Myocardial Infarction (MI): 5,469; ST/T Change (STTC): 5,235; Conduction Disturbance (CD): 4,898; and Hypertrophy (HYP): 2,649. To ensure consistency for our study, only subjects with uniform diagnoses across all trials were included, reducing the dataset to 17,596 subjects. The original 500Hz recordings were downsampled to 250Hz for this study, and standard scaling was applied. Each 10-second recording was segmented into non-overlapping 5-second samples, yielding a total of 57,492 samples. A subject-independent split, respecting patient assignments to folds as recommended by the dataset providers, is used to ensure robust evaluation, with 60\% of subjects assigned to the training set, 20\% to the validation set, and 20\% to the test set.}

\textcolor{black}{For experiments involving LLMs, ChatGPT-4.0 (through the OpenAI GUI \url{https://chatgpt.com/}), the locally deployed Qwen2-VL-7B~\cite{wang2024qwen2} and 
MiMo-VL-7B-RL~\cite{coreteam2025mimovl} were used in their frozen, pre-trained state. \ConMIL's outputs (set-valued predictions and interpretations) were provided as input to these LLMs to guide their visual inspection and decision-making processes.} All implementations are carried out in PyTorch, and experiments were conducted on two NVIDIA RTX 4090 GPUs. The complete implementation of \ConMIL{}, including model training configurations, calibration procedures, and demos for LLM interactions, is publicly available on GitHub\footnote{\url{https://github.com/HuayuLiArizona/Conformalized-Multiple-Instance-Learning-For-MedTS}}.

\subsection{Trust but Verify: LLM Clinical Reasoning with \ConMIL{} Decision Rules}
\label{sec: trust_but_verify}
\begin{figure}
\centering
\includegraphics[width=0.95\textwidth]{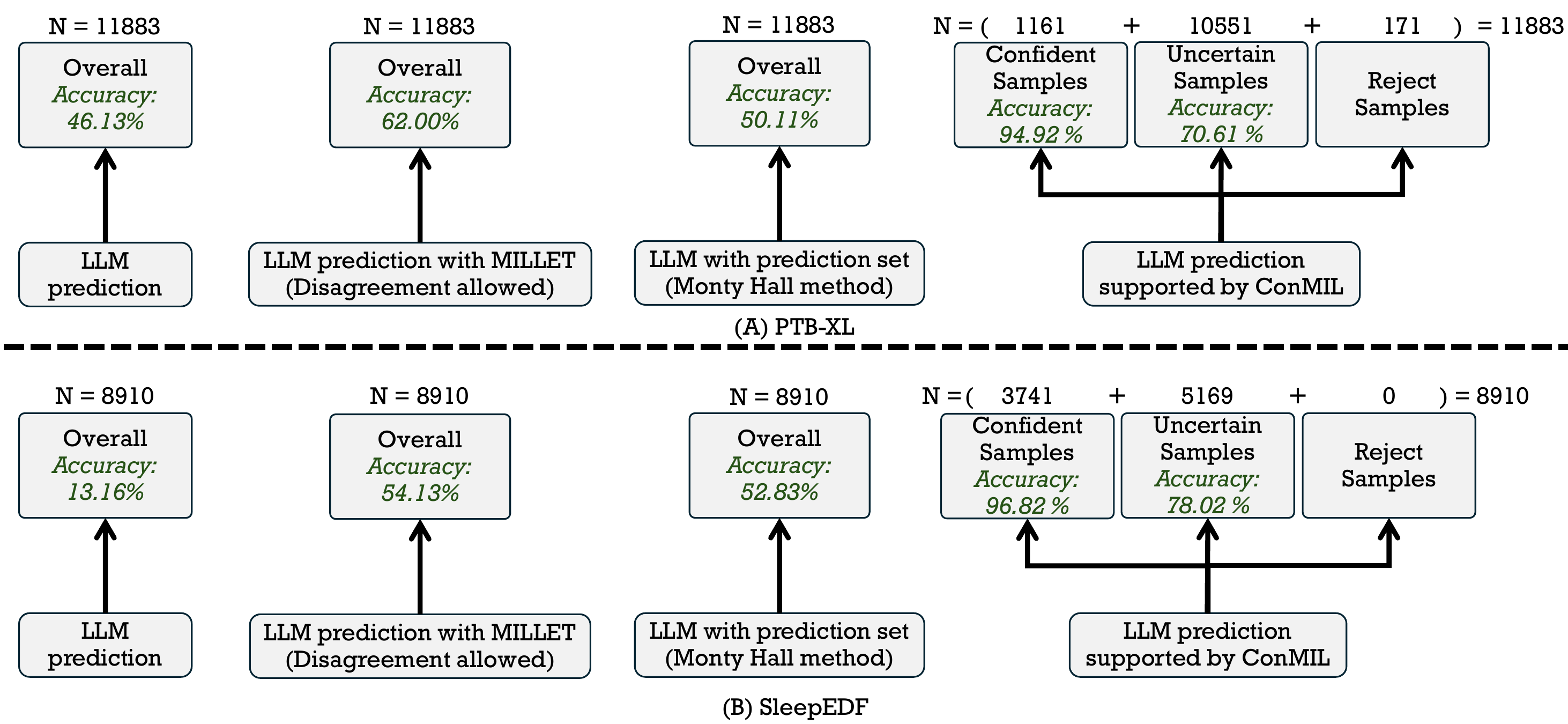}
\caption{\ConMIL{} on \underline{Qwen2-VL-7B}'s predictions for PTB-XL (A) and SleepEDF (B) at $\alpha=0.05$. Qwen2-VL-7B, supported by \ConMIL{}, achieves 94.92\% and 96.82\% accuracy for confident samples and 70.61\% and 78.02\% for uncertain samples, with 171 and 0 samples rejected, respectively. \textcolor{black}{This performance is notably higher than the LLM operating alone, with MILLET support, or with the Monty Hall method, highlighting \ConMIL’s effectiveness in improving both accuracy and reliability through its combination of uncertainty quantification and per-class interpretability. N: number of total samples.}}
\label{fig:qwen}
\end{figure}

\begin{figure}
\centering
\includegraphics[width=0.95\textwidth]{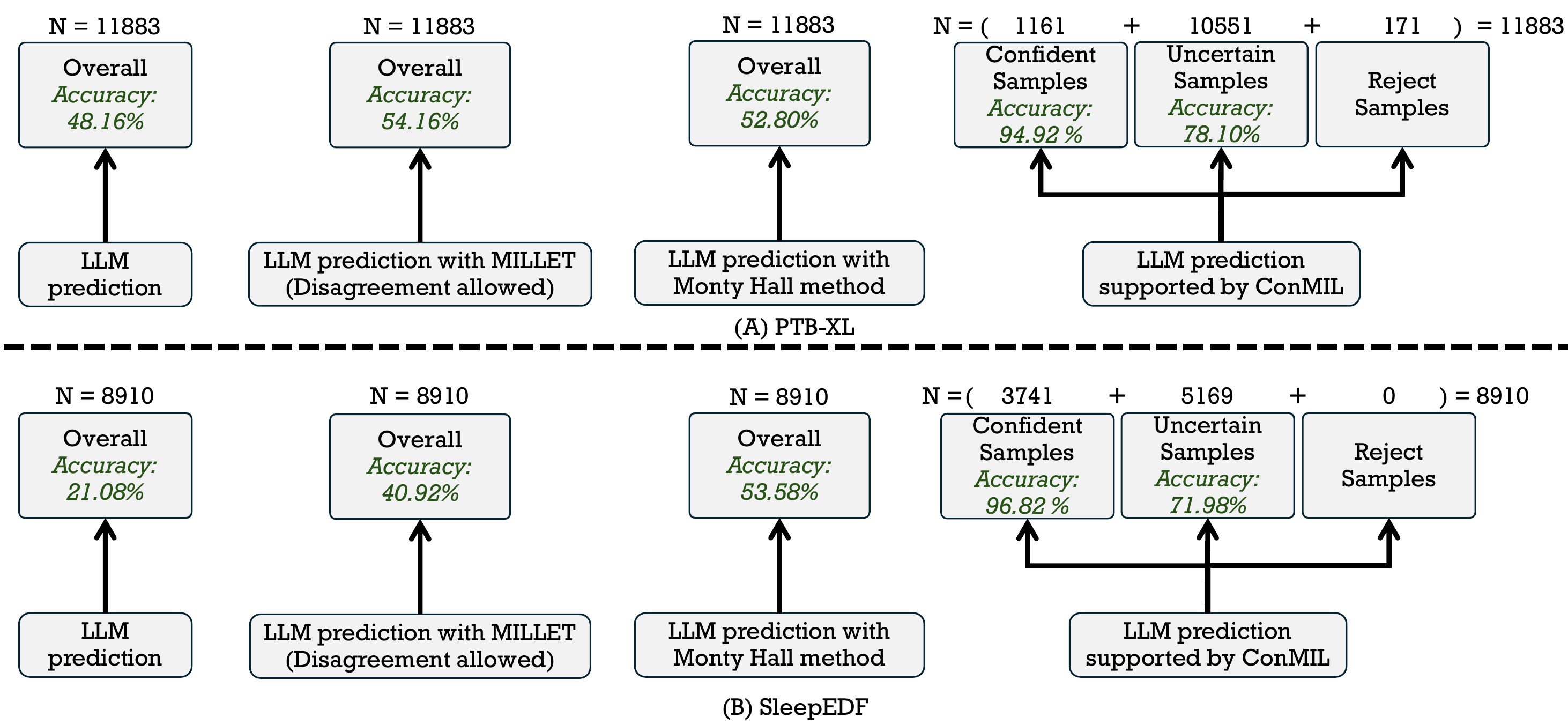}
\caption{\textcolor{black}{\ConMIL{} on \underline{MiMo-VL-7B-RL}'s predictions for PTB-XL (A) and SleepEDF (B) at $\alpha=0.05$. The framework shows consistent benefits with a different LLM. Supported by \ConMIL{}, MiMo-VL-7B-RL attains an accuracy of 94.92\% on confident samples for PTB-XL and 96.82\% for SleepEDF. Accuracy on uncertain samples is 78.10\% (PTB-XL) and 71.98\% (SleepEDF). The number of rejected samples (171 for PTB-XL, 0 for SleepEDF) remains consistent with the conformal calibration. \ConMIL{}'s full support again proves superior to baseline LLM performance and alternative methods like MILLET support or the Monty Hall method, underscoring its robust ability to improve clinical reasoning through confidence-based stratification. N: number of total samples.}}
\label{fig:MIMO}
\end{figure}

We evaluate how \ConMIL{} enhances clinical decision-making capabilities of LLMs, specifically Qwen2-VL-7B \textcolor{black}{and MiMo-VL-7B-RL}, using the full PTB-XL (11883 samples) and SleepEDF (8910 samples) test sets. \textcolor{black}{We compare the performance of these LLMs across four distinct settings: (1) the LLMs operate independently, analyzing only the raw visual plot of the physiological signal; (2) Given the top-1 prediction from the MILLET model and its interpretation, the LLMs are permitted to override the prediction based on independent reasoning. (3) the LLMs are given only the set-valued prediction from conformal prediction and must choose the most likely option from that narrowed-down set, without any accompanying per-class interpretations (Monty Hall method~\cite{vishwakarma2024monty}); (4) the LLMs are fully augmented by \ConMIL's framework, receiving both the set-valued prediction and the per-class interpretability visualizations for each candidate in the set, as detailed in our proposed decision framework. For all settings involving prediction sets, conformal calibration was performed at a confidence level of $\alpha=0.05$. The results are summarized in Figure~\ref{fig:qwen} and Figure~\ref{fig:MIMO}.}

\paragraph{ConMIL Improves Accuracy via Confidence-Based Stratification.}
\textcolor{black}{As shown in the experimental results, \ConMIL’s framework significantly improves the performance of both LLMs by stratifying predictions into confident (singleton prediction set), uncertain (multi-value prediction set), and rejectable categories. For Qwen2-VL-7B, its stand-alone accuracy on the PTB-XL dataset was 46.13\%. With \ConMIL's support, this performance was stratified into a 94.92\% accuracy on the 1161 confident samples and 70.61\% on the 10551 uncertain samples, with 171 instances rejected due to low confidence. On the SleepEDF dataset, its stand-alone performance of 13.16\% was dramatically improved with ConMIL, achieving 96.82\% accuracy on the 3741 confident samples and 78.02\% on the 5169 uncertain ones, with no samples rejected. Similarly, MiMo-VL-7B-RL showed substantial gains. Its stand-alone accuracy of 48.16\% on PTB-XL was stratified by ConMIL into 94.92\% accuracy for confident samples and 78.10\% for uncertain samples. On SleepEDF, its baseline accuracy of 21.08\% was elevated to 96.82\% for confident samples and 71.98\% for uncertain samples. This demonstrates ConMIL's consistent ability to enable different LLMs to deliver high accuracy where appropriate while safely isolating cases that require more nuanced reasoning.}

\paragraph{\textcolor{black}{Comparative Analysis with Alternative Support Methods}}
\textcolor{black}{When compared to other methods of assisting the LLMs, \ConMIL's comprehensive framework is proved to be superior. Providing the LLM with MILLET's top-1 prediction (while allowing disagreement) yields accuracies between 40.92\% and 62.00\% across the models and datasets. While this is an improvement over stand-alone performance, the performance is significantly lower than the accuracy achieved on \ConMIL's confident samples. This approach is limited by MILLET offering interpretability only for its single prediction, lacking the ability to quantify uncertainty or provide reasoning for other potential diagnoses. Simply providing the LLM with the narrowed-down set of options from conformal prediction, without \ConMIL's per-class interpretations, improves accuracy to a range of 50.11\% to 53.58\%. This shows that uncertainty quantification alone is beneficial, but this method's performance is notably weaker than \ConMIL's full support, especially on the uncertain samples. This highlights that the interpretability layer provided by \ConMIL{} is critical for helping the LLM reason effectively when faced with multiple possibilities.}

\paragraph{\textcolor{black}{The Synergy of Interpretability and Uncertainty Quantification}}
\textcolor{black}{These results reveal that \ConMIL{} plays a fundamentally different and more effective role in supporting LLM-based clinical decision-making than alternative approaches. Unlike conventional models that provide a single-point prediction (like MILLET), methods that only provide an uninterpreted set of options (like the Monty Hall method), \ConMIL{} produces confidence-calibrated, set-valued outputs where each candidate in the set is accompanied by its own interpretable explanation.}

\textcolor{black}{This dual provision of calibrated uncertainty and per-class interpretability is what distinguishes \ConMIL{}. When the prediction set contains a single label, it reflects high, verifiable confidence. When multiple labels are included, \ConMIL{} offers the per-class interpretability necessary for the LLM to reason through each possibilities and make an informed decision. \ConMIL{}’s ability to combine both perspectives, providing interpretable support for every class in a confidence-calibrated prediction set is the key to unlocking more accurate, reliable, and transparent LLM-driven clinical reasoning. We also assess the validity of \ConMIL’s conformal prediction framework by examining its ability to provide per-class explanation and achieve the desired coverage guarantees across a range of confidence levels in Supplementary Materials A.}

\subsection{Case Studies: How \ConMIL{} Enhances LLM-Driven Clinical Reasoning}
\label{sec: Case studies}
To further evaluate the practical value of \ConMIL{}, we conducted case studies on PTB-XL and SleepEDF datasets. \textcolor{black}{For the primary qualitative case studies in the main text, we utilized ChatGPT-4.0 due to its established strength in generating exceptionally detailed, narrative-style diagnostic explanations compare to Qwen2-VL-7B. This capability makes it an ideal subject for clearly illustrating the step-by-step impact of \ConMIL's evidence-based guidance on the LLM's reasoning process. To ensure a comprehensive analysis, we provide an example in Supplementary Materials B (Figure S2) for showcasing how to prompt the newest reasoning model MiMo-VL-7B-RL to get structured response. Parallel interviews showcasing the distinct reasoning of MiMo-VL-7B-RL and Qwen2-VL-7B are also provided in Supplementary Materials C.} We examined how ChatGPT4.0 performs under varying levels of interpretability and predictive support. Prior studies have explored the applications of LLMs in cardiovascular~\cite{gunay2024accuracy, zhu2024multimodal, sarraju2023appropriateness} and sleep medicine~\cite{bilal2024enhancing, nakari2024sleep, sano2024exploration}, leveraging their strong visual-language reasoning abilities. However, these models often fall short in domain-specific precision and reliability. Our case studies demonstrate that \ConMIL{} significantly enhances LLMs' diagnostic accuracy by narrowing prediction scope and providing actionable visual explanations.

To systematically explore this enhancement, we selected samples where the top-1 prediction was incorrect, but the set-valued predictions were still included the true label. We designed four progressive interaction settings:
\textcolor{black}{
\begin{enumerate}
    \item \textbf{LLM Stand-Alone:} The model independently analyzes raw input (EEG or ECG) without any external guidance. This simulates a baseline scenario where the LLM must rely solely on its pre-trained knowledge and reasoning capabilities.
    \item \textbf{LLM with Conventional Model Output:} The LLM is provided only with the top-1 prediction from a traditional model (e.g., “REM” or “Myocardial Infarction”), but without any accompanying interpretability. This setting simulates black-box guidance, where the LLM receives a definitive but unexplained recommendation.
    \item \textbf{LLM with \ConMIL{} Set-Valued Prediction:}  The model receives a set of candidate diagnoses (e.g., “REM or N2”) and is tasked with selecting the most likely class. This simulates a scenario where the LLM reasons over a narrowed set of options with explicit uncertainty information, enabling more focused and uncertainty-aware decision-making like the Monty Hall method~\cite{vishwakarma2024monty}.
    \item \textbf{LLM with \ConMIL{} Full Support:} The LLM receives both the set-valued prediction and per-class interpretability visualizations. This allows the model to make an informed choice based not only on the narrowed options but also on feature-level reasoning, simulating a scenario with maximum transparency and support from \ConMIL{}.
\end{enumerate}
}
These settings enable a controlled assessment of how interpretability and uncertainty quantification influence LLM decision-making. 

\begin{figure}
\centering
\includegraphics[width=0.99\textwidth]{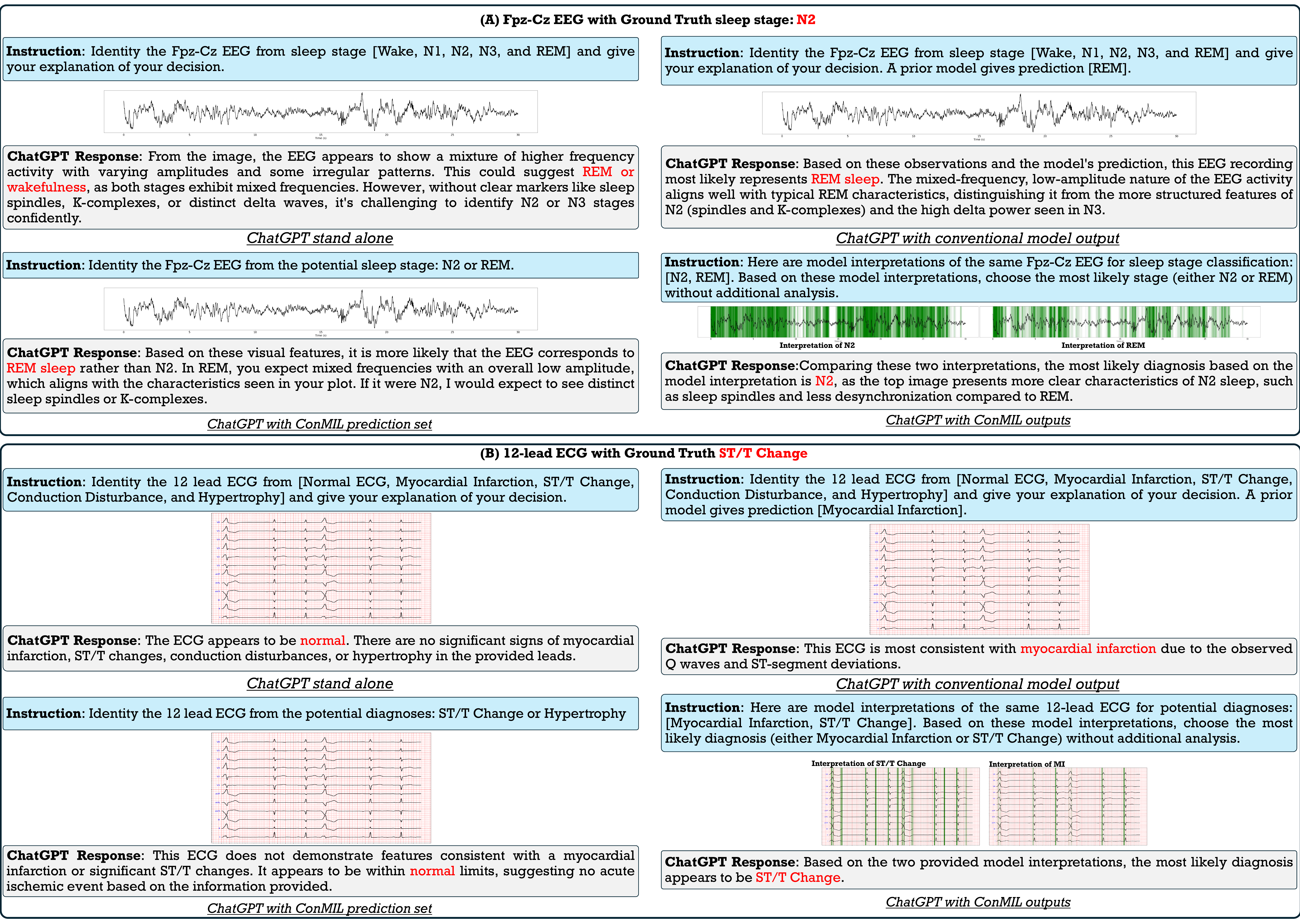}
\caption{Case studies demonstrating how \ConMIL{} enhances ChatGPT4.0’s capabilities in sleep stage classification and arrhythmia classification. \textbf{(A) Sleep Stage Classification}: ChatGPT4.0 correctly identifies the N2 stage only when supported by \ConMIL{}'s set-valued prediction and per-class interpretability. \textbf{(B) Arrhythmia Classification}: For a 12-lead ECG labeled as ST/T Change, ChatGPT4.0 achieves the correct diagnosis only with \ConMIL{}’s interpretability and uncertainty-aware support. These results highlight the critical role of \ConMIL{} in improving diagnostic accuracy via guided reasoning.}
\label{fig:case-studies}
\end{figure}

\paragraph{Case Study 1: Sleep Stage Classification}
We evaluated an Fpz-Cz EEG recording labeled as N2. As shown in Figure~\ref{fig:case-studies} (A), ChatGPT4.0, when unaided, failed to identify the correct stage, misclassifying it as REM or wakefulness due to mixed-frequency patterns. Provided with a top-1 prediction from a conventional model (REM), ChatGPT4.0 echoed the incorrect output. Even with \ConMIL{}'s narrowed set of possible labels (N2, REM), the model still chose incorrectly. Only when given both the prediction set and per-class visualizations—highlighting spindles and reduced desynchronization associated with N2—did ChatGPT4.0 make the correct classification. This example underscores the necessity of class-specific interpretability for LLMs to disambiguate similar diagnostic classes.

\paragraph{Case Study 2: Arrhythmia Classification}
We then tested a 12-lead ECG labeled as ST/T Change. As seen in Figure~\ref{fig:case-studies} (B), ChatGPT4.0 initially misclassified the signal as normal. When prompted with a top-1 model prediction of “Myocardial Infarction,” the model incorrectly followed that path, citing Q waves and ST deviations. Even with the correct label included in a set-valued prediction (e.g., “ST/T Change or Hypertrophy”), the model failed to select it. However, when provided with \ConMIL{}'s interpretability visualization—highlighting relevant temporal regions—ChatGPT4.0 correctly diagnosed ST/T Change, demonstrating the decisive impact of visual cues.

\begin{table}
\caption{Comparison of ChatGPT4.0, Qwen2-VL-7B, and MiMo-VL-7B-RL diagnostic performance across 30 challenging cases (15 for sleep stage classification and 15 for cardiac condition classification) under four settings}
\label{tab:chatgpt}
\centering
\begin{adjustbox}{max width=1\columnwidth}
\begin{tabular}{|c|c|c|c|c|c|}
\hline
Dataset & LLM & Setting 1 & Setting 2 & Setting 3 & Setting 4 \\ \hline
\multirow{3}{*}{PTB-XL}  &ChatGPT4.0   & 0/15 (0\%)      &  0/15 (0\%)        &   9/15 (60\%)               & \boldres{14/15 (93.3\%)}        \\ 
                         &Qwen2-VL-7B & 0/15 (0\%)     &  0/15 (0\%)      &   3/15  (20\%)              & \boldres{13/15 (86.7\%)}        \\ 
                         &MiMo-VL-7B-RL  & 3/15 (20\%)     &  0/15 (0\%)      &   4/15  (26.7\%)              & \boldres{15/15 (100\%)}        \\ \hline
\multirow{3}{*}{SleepEDF}&ChatGPT4.0   & 3/15 (20\%)     &  4/15 (26.6\%)      &   4/15 (26.6\%)                & \boldres{13/15 (86.7\%)}      \\
                         &Qwen2-VL-7B & 2/15 (13.3\%)     &  0/15 (0\%)       &   4/15 (26.6\%)                & \boldres{12/15 (80\%)}      \\ 
                         &MiMo-VL-7B-RL & 2/15 (13.3\%)     &  0/15 (0\%)       &   3/15 (20\%)                & \boldres{12/15 (80\%)}      \\ \hline
\multirow{3}{*}{Total}   &ChatGPT4.0   & 3/30 (10\%)     &  4/30 (13\%)     &   13/30 (43.3\%)              & \boldres{27/30 (90\%)}     \\ 
                         &Qwen2-VL-7B & 2/30 (6.7\%)     &  0/30  (0\%)       &   7/30 (23.3\%)               & \boldres{25/30 (83.3\%)}     \\ 
                         &MiMo-VL-7B-RL & 5/30 (16.7\%)     &  0/30  (0\%)       &   7/30 (23.3\%)               & \boldres{27/30 (90\%)}     \\ \hline
\end{tabular}
\end{adjustbox}
\end{table}

\paragraph{Quantitative Evaluation}
\textcolor{black}{To further validate these findings, we conducted a quantitative analysis on 30 randomly selected challenging cases—15 from sleep staging and 15 from arrhythmia classification. In each case, the top-1 prediction was incorrect, but the correct label was present within the prediction set. As shown in Table~\ref{tab:chatgpt}, ChatGPT4.0, Qwen2-VL-7B, and MiMo-VL-7B-RL all showed substantial improvements when supported by \ConMIL{}. ChatGPT4.0 achieved 27 correct diagnoses with full \ConMIL{} support, compared to only 3 in the stand-alone setting. Similarly, Qwen2-VL-7B improved from 2 to 25 correct diagnoses and MiMo-VL-7B-RL improved from 5 to 27. These results confirm that \ConMIL{} meaningfully improves diagnostic reasoning through confidence-aware and interpretable guidance.}

\subsection{From Ambiguity to Insight: Interactive Reasoning with LLMs and \ConMIL{}}
\label{appendix:Going-deeper}

\begin{figure*}
\centering
\includegraphics[width=0.99\linewidth]{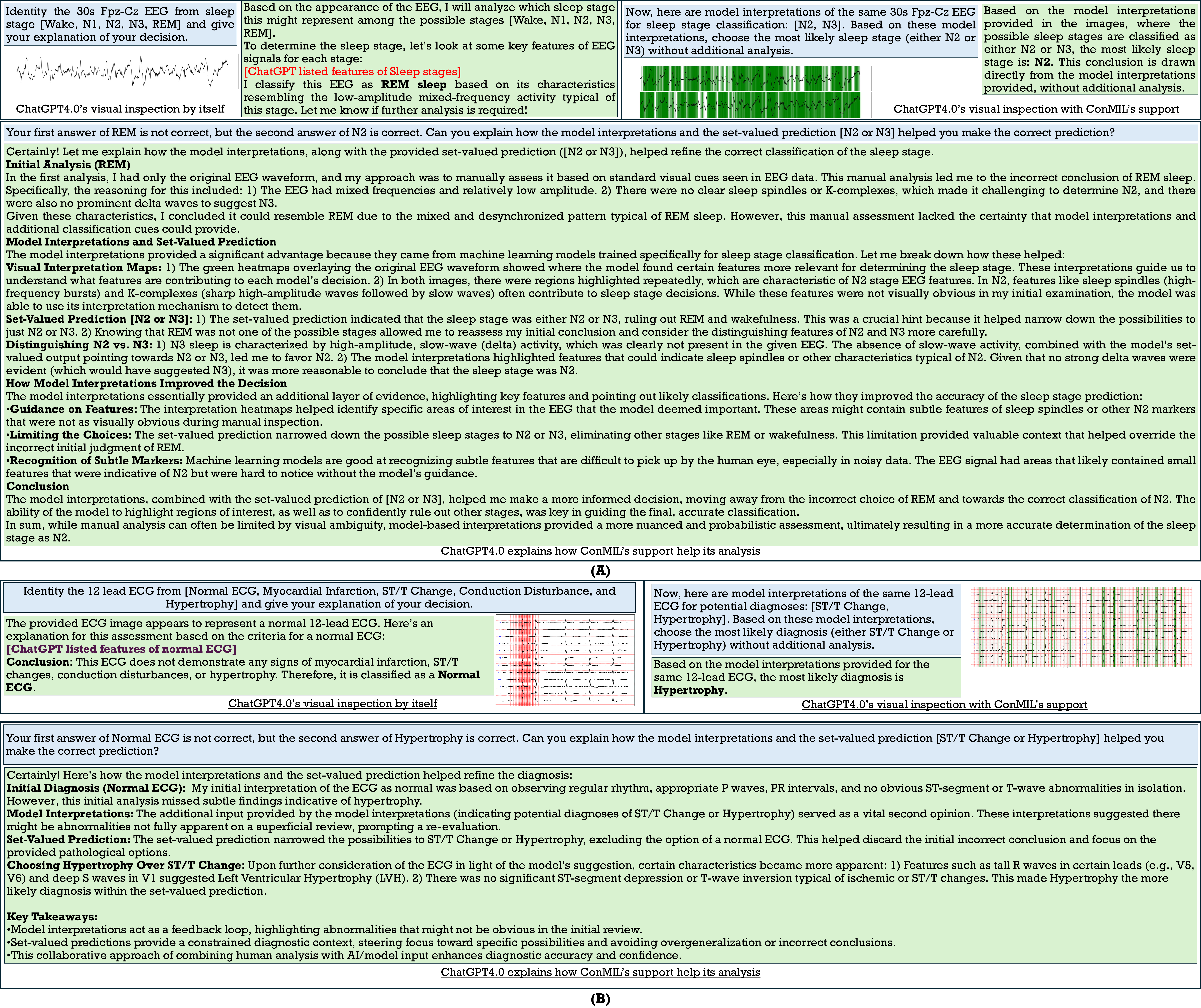}
\caption{“Interview” of ChatGPT4.0 on EEG sleep stage classification (A) and ECG arrhythmia classification (B), illustrating how \ConMIL{} refines diagnostic reasoning. \ConMIL{}'s interpretability and set-valued predictions guide ChatGPT4.0 by narrowing potential diagnoses and highlighting overlooked clinical features (e.g., sleep spindles in EEG or hypertrophy waveforms in ECG). Green boxes show ChatGPT4.0’s responses; blue boxes show the instructions.}
\label{fig:chatgpt-interview}
\end{figure*}

Figure~\ref{fig:chatgpt-interview} presents an interactive “interview” with ChatGPT4.0, highlighting how \ConMIL{} enhances visual inspection and clinical decision-making across two tasks: EEG-based sleep stage classification and ECG-based arrhythmia diagnosis. The integration of \ConMIL{}’s interpretability overlays and set-valued predictions reveals its impact on guiding LLMs toward more accurate and clinically grounded conclusions. \textcolor{black}{We also provide interview of Qwen2-VL-7B, and MiMo-VL-7B-RL in Supplementary Materials C.}

\paragraph{Initial Assessment and Challenges.}
In the first task, ChatGPT4.0 was asked to classify a 30-second Fpz-Cz EEG segment. Initially, it incorrectly labeled the stage as REM sleep, citing low-amplitude mixed-frequency activity. However, it overlooked critical N2 indicators—such as sleep spindles and K-complexes—leading to uncertainty. As it noted: “I classify this EEG as REM sleep based on its resemblance to low-amplitude mixed-frequency activity typical of that stage.” Similarly, in the ECG-based task, ChatGPT4.0 initially classified a 12-lead ECG as normal, failing to identify signs of hypertrophy. The model focused on generalized waveform patterns and missed deeper feature-specific cues critical to diagnosis.

\paragraph{Refinement Enabled by \ConMIL{}.}
The introduction of \ConMIL{} fundamentally changed the LLM’s reasoning. By providing structured, class-specific visual cues and reducing diagnostic ambiguity through set-valued predictions, ChatGPT4.0 was able to revise its initial decisions:
\begin{itemize}
    \item \textbf{Sleep Stage Classification:} \ConMIL{} highlighted characteristic sleep spindles and K-complexes, helping to differentiate N2 from REM or wakefulness. The prediction set was also narrowed to N2 and N3, guiding ChatGPT4.0 to correctly choose N2. As the model reflected: “The set-valued prediction ruled out REM and wakefulness. This crucial hint helped reassess my initial conclusion and focus on N2-specific features.”
    \item \textbf{Arrhythmia Diagnosis:} For the ECG, \ConMIL{} identified wave progression abnormalities, such as deep S waves and prominent R-wave growth—typical of hypertrophy. The exclusion of “normal” from the prediction set prompted ChatGPT4.0 to re-evaluate its assessment: “The set-valued prediction excluded a normal ECG, helping me recognize the hypertrophy-specific wave patterns, such as deep S waves in lead V1 and exaggerated R-wave progression in V5.”
\end{itemize}

\paragraph{Interpretability and Clinical Trust.}
Beyond improving accuracy, \ConMIL{} added a crucial layer of interpretability. ChatGPT4.0 explicitly acknowledged the value of visual cues in clarifying subtle features, such as N2 spindles, that are often hard to detect in noisy signals. As the model stated: “Machine learning models are adept at recognizing subtle features that can be difficult for the human eye to discern.” By surfacing these cues through heatmaps, \ConMIL{} supports traceability and transparency—key prerequisites for trust in clinical AI tools.

\paragraph{Synergy Between Domain Expertise and General Reasoning.}
These results highlight the complementary strengths of \ConMIL{} and LLMs. While \ConMIL{} excels at extracting and highlighting domain-specific signal patterns, LLMs like ChatGPT leverage general medical knowledge to contextualize those cues. For instance, \ConMIL{}'s detection of spindles helped ChatGPT4.0 validate an N2 classification. Likewise, hypertrophy-indicative waveforms led the model to articulate relevant clinical implications. This cooperative dynamic enables a level of diagnostic precision that neither system could achieve alone—especially in ambiguous or borderline cases.

\section{Discussion}
Through rigorous experimentation across two critical clinical tasks, sleep stage classification and arrhythmia detection. \ConMIL{} demonstrated its ability to substantially enhance the diagnostic precision of LLMs such as Qwen2-VL-7B,MiMo-VL-7B-RL and ChatGPT4.0. For example, Qwen2-VL-7B and MiMo-VL-7B-RL supported by \ConMIL{} achieved significantly higher accuracy on both the PTB-XL and SleepEDF datasets, while exhibiting more cautious risk stratification through set-valued predictions and rejection options.

Importantly, case studies revealed that ChatGPT4.0, when aided by \ConMIL{}'s set-valued outputs and per-class interpretability, was often able to arrive at the correct diagnosis—even in cases where \ConMIL{}'s own top-1 prediction was incorrect. This addresses a fundamental limitation of conventional SSMs, which typically offer single-point predictions without the flexibility for nuanced diagnostic reasoning or fail-safe mechanisms. By augmenting LLMs with interpretable and calibrated support, \ConMIL{} enables task-specific precision to complement the broad contextual reasoning of general-purpose models.

The “interview” protocol we conducted with ChatGPT4.0 further illustrates how \ConMIL{} supports clinical reasoning. This structured interaction loop encouraged the LLM to evaluate multiple diagnostic hypotheses, guided by visual evidence and confidence-aware predictions. Our findings align with recent trends in augmenting LLMs through external tools such as chain-of-thought prompting~\cite{wei2022chain} and retrieval-augmented generation~\cite{lewis2020retrieval}.
Not only does this illustrates \ConMIL{}’s utility as a decision-support module, but also it suggests new ways of scaffolding LLM behavior, such as ``interview'' based reasoning for robust and explainable clinical workflows.

\subsection{Real-World Clinical Impacts}
\textcolor{black}{Our findings highlight \ConMIL's transformative potential in real-world healthcare settings. By providing per-class explanations and calibrated uncertainty, \ConMIL{} addresses the "black box" problem in medical AI, fostering clinician adoptions and enabling more effective human-AI collaboration. This approach is crucial in high-stakes settings where interpretability and reliability are paramount.}

\textcolor{black}{The practical benefits are evident across various specialties. For example, in cardiology, LLMs supported by \ConMIL{} can identify ST/T abnormalities and other critical cardiac conditions like hypertrophy with greater accuracy and interpretability. \ConMIL{} facilitates this by directing the LLM's attention to specific indicative waveforms within ECGs. This enables more timely triage and intervention, potentially preventing critical escalations such as myocardial infarction. These capabilities are particularly impactful in remote or telehealth scenarios, where \ConMIL{}-enhanced LLMs can provide rapid and reliable interpretations, extending expert-level assistance to underserved regions and potentially saving lives. The framework's ability to flag ambiguous cases for clinicians review further ensures patient's safety in such settings. Similarly, in sleep medicine, the proposed framework enables more efficient and interpretable sleep scoring. The success in these areas suggests potential for generalizing \ConMIL{} to other diagnostic signal interpretation tasks, such as intraoperative neurophysiological monitoring or fetal heart rate monitoring, motivating further validation in these diverse clinical areas.}

\textcolor{black}{From a systems perspective, \ConMIL{} redefines LLM augmentation by positioning SSMs as providers of structured, interpretable, and confidence-calibrated evidence rather than just factual boosters. This fosters transparent, reject-aware, and clinically aligned AI decision-making. LLMs leverage \ConMIL{}'s outputs to generate actionable insights, guided by domain-specific expertises while maintaining their general reasoning capabilities. The "interview" protocol demonstrated with ChatGPT4.0 further suggests fertile ground for research into advanced LLM reasoning strategies and robust, explainable clinical workflows. Ultimately, \ConMIL{} lays the groundwork for scalable and trustworthy AI-assisted diagnostics that balance automation with essential human oversight, representing a key step towards safer and more effective patient-centered healthcare. Furthermore, by leveraging pre-trained LLMs augmented by specialized SSMs like \ConMIL{}, healthcare institutions may achieve robust AI capabilities more cost-effectively than resource-intensive finetuning of large models for every specific task.}

\subsection{Limitations}

While \ConMIL{} markedly improves LLMs’ capabilities in the visual inspection of physiological data, several limitations warrant consideration. First, this framework is tailored primarily for structured physiological signals such as ECG and EEG. However, real-world clinical decision-making often involves multimodal data sources, including unstructured clinical notes, imaging, and electronic health records (EHRs), which are not yet addressed in this work.

Second, \ConMIL{} is not explicitly designed to support structured clinical report generation. Although it enhances diagnostic reasoning and interpretability, it does not produce outputs aligned with documentation standards or clinical guidelines, an essential requirement for deployment in formal medical settings.

Third, in the context of broader medical time-series data, not limited to physiological signals—modalities such as accelerometry-based motion data. do not always lend themselves well to waveform-based visual inspection. For some modalities, alternative representations—such as spectrograms, embeddings, or frequency-domain features, may offer more meaningful insights. The current framework may therefore be limited in scope for these scenarios.

\subsection{Future Work}
Future research will seek to extend \ConMIL{} into truly multimodal diagnostic systems capable of jointly processing structured physiological signals and unstructured clinical narratives. While the current implementation focuses on waveform-based visual inspection of physiological signals such as ECG and EEG, real-world clinical reasoning often requires integrating contextual information, including patient history, physician notes, lab results, and imaging reports. Bridging this gap would allow for a more comprehensive understanding of the patient state and enable the AI system to support longitudinal, cross-modal reasoning that more closely mirrors the clinical diagnostic process.

Another important direction involves enhancing the LLM’s ability to generate structured outputs that align with established clinical guidelines, such as the American Academy of Sleep Medicine (AASM) scoring manual~\cite{berry2012aasm} for sleep staging or the American Heart Association (AHA) for cardiovascular diagnostics~\cite{joglar20242023}. This would increase the system's readiness for integration into clinical workflows by ensuring outputs are not only interpretable but also actionable and legally traceable. Incorporating templated or schema-aware report generation into the LLM prompting pipeline could enable automatic generation of structured assessments and plans, reducing clinician workload while preserving documentation quality.

To further expand its generalizability, \ConMIL{} should be adapted to support a broader spectrum of input representations beyond raw waveforms. For certain modalities or clinical use cases, alternative forms such as spectrograms, time-frequency decompositions, or learned latent embeddings may offer more informative or noise-robust features. Supporting these formats would allow \ConMIL{} to handle more diverse and potentially degraded input data, improving performance in real-world scenarios such as ambulatory monitoring or low-resource environments.

Furthermore, the "interview" protocol introduced in this study offers a promising foundation for more formalized chain-of-thought prompting strategies. By structuring LLM interactions to explicitly consider multiple hypotheses and iteratively evaluate model-generated interpretations, this approach can promote transparency and error detection. Future work could explore the design of dynamic dialogue agents that use \ConMIL{}’s interpretive signals as anchors in a reasoning loop—enabling the LLM to self-verify, revise, or defer decisions based on confidence and evidence. Such mechanisms could be instrumental in building AI systems that are not only diagnostically accurate but also self-aware of their uncertainty and reasoning limitations, paving the way toward safer and more collaborative human-AI decision-making in healthcare. \textcolor{black}{Finally, exploring parameter-efficient fine-tuning techniques like LoRA for the LLMs, informed by ConMIL's outputs, could be a promising avenue to further enhance their domain-specific reasoning capabilities without the need for full model retraining.}

\section{Conclusion}
This work represents a significant step forward in healthcare AI by demonstrating how SSMs and LLMs can be synergistically combined to enhance clinical decision-making—particularly in the context of physiological signal analysis. While the development of SSMs has plateaued in terms of raw accuracy, \ConMIL{} shows that these models retain immense value as interpretable, reliable, and task-specific components within broader AI systems. Rather than positioning SSMs and LLMs as competing paradigms, \ConMIL{} reframes SSMs as foundational decision-support modules that complement the reasoning capabilities of LLMs. This integration not only improves diagnostic accuracy but also addresses critical concerns around interpretability, uncertainty, and clinical trust. By offering per-class explanations and confidence-calibrated predictions, \ConMIL{} elevates the transparency and safety of LLM-driven diagnostics—hallmarks of responsible AI deployment in medicine. The successful fusion of \ConMIL{} with LLMs underscores the broader promise of hybrid AI architectures in healthcare. SSMs bring specialization, efficiency, and low computational overhead, while LLMs contribute generalizability, multimodal understanding, and contextual reasoning. Together, they form a robust and adaptable ecosystem capable of addressing the complexities of real-world clinical workflows. Looking ahead, the core principles of \ConMIL{} extend beyond physiological signal analysis. Its interpretability-driven, confidence-aware framework could be adapted to domains such as pathology, radiology, or genomics, where transparency is equally vital. Ultimately, this approach lays the groundwork for collaborative AI systems that are not only accurate, but also explainable, trustworthy, and seamlessly integrated into clinical practice—paving the way for safer, more effective, and patient-centered healthcare delivery.

\section{Declarations}
\subsection{Ethical Approval}
Not Applicable

\subsection{Funding}
This work was supported by a grant from the National Science Foundation (\#2052528).

\subsection{Data availability statements}
All data used in this paper are publicly available on PhysioNet (PTB-XL:\url{https://physionet.org/content/ptb-xl/1.0.0/}, SleepEDF:\url{https://physionet.org/content/sleep-edfx/1.0.0/}). 

\clearpage
\bibliography{ref}
\bibliographystyle{unsrt}
\clearpage

\appendix
% Reset figure and table counters
\setcounter{figure}{0}
\setcounter{table}{0}

% Change numbering to include 'A' prefix
\renewcommand{\thefigure}{A\arabic{figure}}
\renewcommand{\thetable}{A\arabic{table}}

\section{Supplementary Materials: Stagnation of SSMs and \ConMIL{} Conformal Coverage}
\label{app:stagnation}
\textbf{Baselines and experimental setup}: We compare the performance of \ConMIL{} against several baseline SSMs. These include conventional convolutional neural networks (CNNs) such as InceptionTime~\cite{ismail2020inceptiontime}, and recent advanced transformer architecture (MedFormer~\cite{wang2024medformer}). In addition, we evaluate multiple MIL-based models, including ABMIL~\cite{ilse2018attention} and DSMIL~\cite{li2021dual} which were originally designed for Whole Slide Imaging (WSI), and recently proposed MILLET~\cite{early2024inherently} and TimeMIL~\cite{chen2024timemil}, which are tailored for time series data. 

All models are trained using binary cross-entropy (BCE) loss and the AdamW optimizer~\cite{loshchilov2017decoupled}, with a fixed learning rate of 0.005, weight decay of 0.001, and $\beta_1=0.9,\beta_2=0.999$ for momentum adaptation. Training is conducted for 40 epochs using a batch size of 128, and the model is evaluated at each epoch. Early stopping with a patience of 5 epochs is used to prevent overfitting. Model selection is based on the highest F1 score achieved on the validation set. We report the mean and standard deviation over three independent runs to ensure robustness. To address potential performance degradation due to data splitting in the split-conformal prediction setting, we calibrate \ConMIL{} on the validation set, ensuring fair comparisons with baseline methods. The CRC implementation is adapted from the GitHub repository\footnote{\url{https://github.com/aangelopoulos/conformal-prediction}} of \cite{angelopoulos2021gentle}, with modifications to accommodate per-class FNR constraints.

As shown in Table~\ref{tab:milnet-performance}, recent models—including advanced transformer-based architectures such as MedFormer—yield only marginal improvements over earlier baselines like InceptionTime. \textcolor{black}{This trend of marginal gains underscores a broader stagnation in SSM development when success is measured solely by standalone predictive accuracy, motivating a fundamental shift in perspective. While our \ConMIL{} model, evaluated as a standalone SSM, achieves competitive performance and even outperforms other state-of-the-art baselines, our primary objective was not to secure another minor improvement in these metrics. Instead, we aimed to break from the loop of incremental advancements by directly addressing the inherent limitations—such as a lack of deep interpretability and calibrated uncertainty—that have traditionally prevented SSMs from being reliably integrated into more complex, intelligent clinical systems. Therefore, \ConMIL{} was designed to demonstrate how an SSM can be reimagined: transformed from an isolated predictor into a reliable, supportive component that provides actionable, evidence-based insights to a LLM. Its 'outperformance' is consequently measured by its success in these new dimensions: its ability to function as a trustworthy plug-in and enable a more sophisticated and accurate reasoning process within a larger AI framework.}

To illustrate \ConMIL{}’s interpretability, we compare its output to that of MILLET in Figure~\ref{fig:compare}. MILLET produces post-hoc interpretations that only explain the final prediction, whereas \ConMIL{} generates per-class interpretability maps. These maps highlight which segments of the time series contribute to each possible diagnosis, enabling more transparent and actionable insights for downstream decision-makers. 

Importantly, this architectural design also leads to reduced class-wise variability in the latent space. As formalized in Theorem 2, conditioning attention on per-class tokens lowers class-conditioned entropy relative to using a single global token. This theoretical result suggests that \ConMIL{} promotes tighter clustering of instances within each class, simplifying the decision boundaries and improving classification reliability. The per-class interpretability heatmaps shown in Figure~\ref{fig:compare} visually reinforce this effect.

\begin{table*}
    \caption{Performance comparison of baseline models based on several evaluation metrics: Accuracy, Precision, Recall, F1 Score, AUROC, and AUPRC. Models compared include MedFormer~\cite{wang2024medformer}, InceptionTime~\cite{ismail2020inceptiontime}, TSMixer~\cite{chen2023tsmixer}, ABMIL~\cite{ilse2018attention}, DSMIL~\cite{li2021dual}, MILLET~\cite{early2024inherently}, TimeMIL~\cite{chen2024timemil}, and \ConMIL{} (SSM Only). The table also highlights the interpretability provided by each model, distinguishing between models with no interpretability, partial interpretability, and per-class interpretability.}
    \label{tab:milnet-performance}
    \centering
    \begin{adjustbox}{max width=1.0\columnwidth}
    \begin{tabular}{c|c|c|cccccc}
    \toprule
    Datasets                                          & Interpretability  & Models    & Accuracy       & Precision      & Recall         & F1 score       & AUROC          & AUPRC             \\\midrule
    \multicolumn{1}{c|}{\multirow{6}{*}{PTBXL}}      & No    & MedFormer  & 72.87$\pm$0.23 & 64.14$\pm$0.42 & 60.60$\pm$0.46 & 62.02$\pm$0.37 & 89.66$\pm$0.13 & 66.39$\pm$0.22 \\
    \multicolumn{1}{c|}{}                            & No  &InceptionTime & 74.47$\pm$0.57 & 60.70$\pm$2.48 & 67.38$\pm$3.31 & 60.13$\pm$1.57 & 87.09$\pm$1.44 & 68.01$\pm$2.11   \\
    \multicolumn{1}{c|}{}                            & No  &TSMixer & 69.18$\pm$0.45 & 61.66$\pm$1.82 & 56.68$\pm$2.43 & 56.93$\pm$1.22 & 86.70$\pm$1.32 & 60.85$\pm$1.91   \\
    \multicolumn{1}{c|}{}                            & Partial    & DSMIL     & 74.32$\pm$0.66 & 63.21$\pm$1.31 & 66.48$\pm$1.27 & 61.21$\pm$1.27 & 86.81$\pm$0.46 & 67.33$\pm$0.63       \\
    \multicolumn{1}{c|}{}                            & Partial    & ABMIL     & \boldres{74.79$\pm$0.32} & 62.32$\pm$1.42 & \boldres{67.81$\pm$1.29} & 60.92$\pm$1.27 & 87.86$\pm$0.45 & \boldres{68.93$\pm$0.57}       \\
    \multicolumn{1}{c|}{}                            & Partial    & MILLET    & 74.55$\pm$0.43 & 62.19$\pm$1.04 & 68.22$\pm$1.51 & 60.40$\pm$0.55 & 87.64$\pm$0.34 & 68.77$\pm$0.78       \\
    \multicolumn{1}{c|}{}                            & Partial    & TimeMIL   & 73.88$\pm$0.88 & 60.65$\pm$3.04 & 64.73$\pm$3.00 & 59.93$\pm$1.99 & 85.80$\pm$1.52 & 66.05$\pm$2.24    \\
    \multicolumn{1}{c|}{}                            & \boldres{Per-class}    & \ConMIL{} (SSM Only)      & 74.28$\pm$0.22 & \boldres{66.01$\pm$0.97} & 61.17$\pm$0.72 & \boldres{62.12$\pm$0.48} & \boldres{89.92$\pm$0.16} & 67.42$\pm$0.33 \\\midrule
    \multicolumn{1}{c|}{\multirow{6}{*}{sleepEDF}}   & No    & MedFormer  & 82.77$\pm$0.86 & 71.44$\pm$0.92 & 74.65$\pm$0.56 & 71.12$\pm$0.76 & 93.74$\pm$0.43 & 79.89$\pm$0.66 \\
    \multicolumn{1}{c|}{}                            & No  &InceptionTime & 84.21$\pm$1.09 & 74.83$\pm$1.62 & 76.95$\pm$0.62 & 73.84$\pm$2.09 & 95.01$\pm$0.55 & 81.29$\pm$0.64    \\
    \multicolumn{1}{c|}{}                            & No  &TSMixer & 60.15$\pm$1.72 & 59.99$\pm$2.81 & 47.04$\pm$3.31 & 44.93$\pm$1.57 & 81.92$\pm$2.35 & 49.82$\pm$2.71   \\
    \multicolumn{1}{c|}{}                            & Partial    & DSMIL     & 84.91$\pm$0.41 & 75.01$\pm$0.64 & 77.23$\pm$0.80 & 74.41$\pm$0.68 & 95.26$\pm$0.19 & 81.61$\pm$0.24      \\
    \multicolumn{1}{c|}{}                            & Partial    & ABMIL     & 85.51$\pm$0.43 & 75.22$\pm$0.68 & 77.30$\pm$0.84 & 74.71$\pm$0.77 & 95.36$\pm$0.19 & 81.80$\pm$0.18      \\
    \multicolumn{1}{c|}{}                            & Partial    & MILLET    & 85.23$\pm$0.38 & 74.85$\pm$0.60 & 77.11$\pm$0.71 & 74.35$\pm$0.60 & 95.24$\pm$0.18 & 81.61$\pm$0.56      \\
    \multicolumn{1}{c|}{}                            & Partial    & TimeMIL   & 84.74$\pm$0.80 & 75.82$\pm$0.98 & \boldres{77.72$\pm$0.74} & 74.87$\pm$1.25 & 95.20$\pm$0.40 & 81.69$\pm$0.35  \\
    \multicolumn{1}{c|}{}                            & \boldres{Per-class}    & \ConMIL{} (SSM Only)      & \boldres{85.82$\pm$0.37} & \boldres{78.06$\pm$0.27} & 75.65$\pm$0.32 & \boldres{76.78$\pm$0.12} & \boldres{96.33$\pm$0.24} & \boldres{81.91$\pm$0.23} \\\bottomrule 
    \end{tabular}
    \end{adjustbox}
\end{table*}

\begin{figure}[H]
    \centering
    \includegraphics[width=0.98\linewidth]{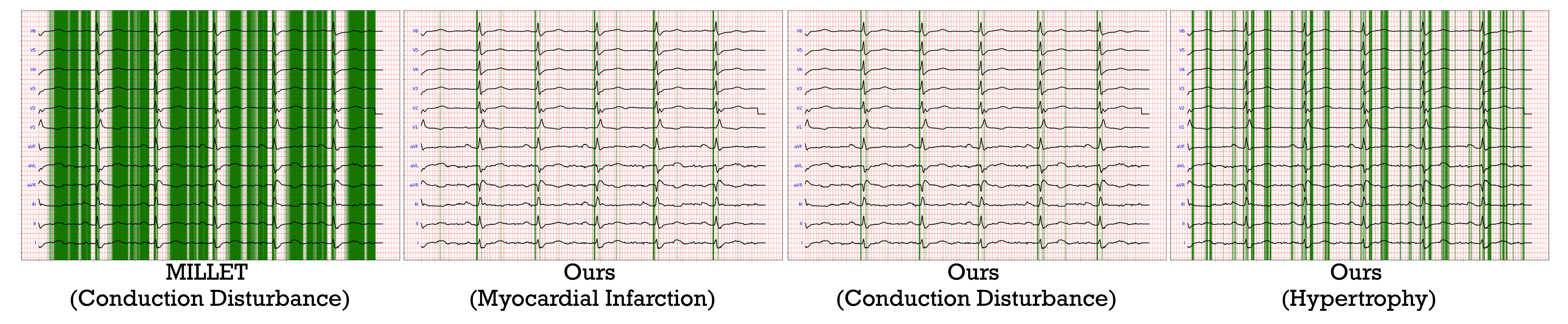}
    \caption{Comparison of MILLET and \ConMIL{} interpretations. MILLET provides interpretation only for the final prediction, whereas \ConMIL{} generates per-class interpretations, offering a more granular and transparent understanding of how different time-series segments contribute to each potential diagnosis.}
    \label{fig:compare}
\end{figure}

The effectiveness of \ConMIL{}’s QTrans-Pooling is also evident in its computational efficiency. On the PTB-XL dataset with a batch size of 8, \ConMIL{} achieves an average inference time of 13.169 ms per batch, compared to 14.723 ms for TimeMIL. This improvement is attributed to the use of cross-attention over multiple class-specific tokens rather than self-attention over all time steps, reducing the attention complexity from $\mathcal{O}(T^2d)$ to $\mathcal{O}(TKd)$

\begin{table*}
    \caption{Conformal coverage evaluation under varying confidence levels ($\alpha = 0.1, 0.05, 0.025, 0.01$). The table reports Per-Class FNR for different class labels, along with Marginal Coverage percentages that reflect how often the true label is included in the prediction set. Additionally, the Average Size of the prediction set is shown, representing the number of labels included on average.}
    \label{tab:coverage}
    \centering
    \begin{adjustbox}{max width=1.0\columnwidth}
    \begin{tabular}{c|c|cccc|ccccc|ccccc}
    \toprule
    \multicolumn{2}{c|}{Dataset}                            & \multicolumn{5}{c|}{PTBXL}                                                                                                                                             & \multicolumn{5}{c}{sleepEDF}                                                                                                \\ \midrule
    \multicolumn{2}{c|}{$\alpha$}                           & \multicolumn{1}{c|}{Before cal}            & \multicolumn{1}{c|}{0.1}            & \multicolumn{1}{c|}{0.05}          & \multicolumn{1}{c|}{0.025}           & 0.01    & \multicolumn{1}{c|}{Before cal}           & \multicolumn{1}{c|}{0.1}           & \multicolumn{1}{c|}{0.05}          & \multicolumn{1}{c|}{0.025}         & 0.01          \\ \midrule
    \multicolumn{1}{c|}{\multirow{5}{*}{Per-Class FNR}} & 1 & \multicolumn{1}{c|}{8.58$\pm$2.06} & \multicolumn{1}{c|}{10.58$\pm$0.32} & \multicolumn{1}{c|}{5.46$\pm$0.30}  & \multicolumn{1}{c|}{2.82$\pm$0.20}  & \multicolumn{1}{c|}{1.04$\pm$0.11}  & \multicolumn{1}{c|}{7.23$\pm$2.64}  & \multicolumn{1}{c|}{6.21$\pm$0.53}  & \multicolumn{1}{c|}{3.05$\pm$0.19}  & \multicolumn{1}{c|}{1.80$\pm$0.25}  & 0.75$\pm$0.25  \\ 
    \multicolumn{1}{c|}{}                               & 2 & \multicolumn{1}{c|}{28.24$\pm$1.79} & \multicolumn{1}{c|}{12.80$\pm$0.46} & \multicolumn{1}{c|}{6.31$\pm$0.65}  & \multicolumn{1}{c|}{3.23$\pm$0.42}  & \multicolumn{1}{c|}{1.21$\pm$0.26}  & \multicolumn{1}{c|}{77.69$\pm$2.22}  & \multicolumn{1}{c|}{7.86$\pm$2.34}  & \multicolumn{1}{c|}{2.69$\pm$1.17}  & \multicolumn{1}{c|}{1.29$\pm$0.59}  & 0.54$\pm$0.29  \\ 
    \multicolumn{1}{c|}{}                               & 3 & \multicolumn{1}{c|}{31.93$\pm$4.48} & \multicolumn{1}{c|}{9.59$\pm$0.41}  & \multicolumn{1}{c|}{4.87$\pm$0.33}  & \multicolumn{1}{c|}{2.81$\pm$0.26}  & \multicolumn{1}{c|}{1.09$\pm$0.17}  & \multicolumn{1}{c|}{5.89$\pm$0.65}  & \multicolumn{1}{c|}{5.58$\pm$0.61}  & \multicolumn{1}{c|}{2.57$\pm$0.35}  & \multicolumn{1}{c|}{1.41$\pm$0.28}  & 0.59$\pm$0.22  \\ 
    \multicolumn{1}{c|}{}                               & 4 & \multicolumn{1}{c|}{35.18$\pm$5.33} & \multicolumn{1}{c|}{9.27$\pm$0.50}  & \multicolumn{1}{c|}{4.20$\pm$0.27}  & \multicolumn{1}{c|}{2.06$\pm$0.32}  & \multicolumn{1}{c|}{0.67$\pm$0.12}  & \multicolumn{1}{c|}{10.09$\pm$3.70}  & \multicolumn{1}{c|}{8.32$\pm$1.31}  & \multicolumn{1}{c|}{4.61$\pm$0.73}  & \multicolumn{1}{c|}{2.16$\pm$0.49}  & 0.71$\pm$0.13  \\ 
    \multicolumn{1}{c|}{}                               & 5 & \multicolumn{1}{c|}{76.63$\pm$6.07} & \multicolumn{1}{c|}{10.92$\pm$0.91} & \multicolumn{1}{c|}{5.48$\pm$0.67}  & \multicolumn{1}{c|}{2.59$\pm$0.50}  & \multicolumn{1}{c|}{1.11$\pm$0.26}  & \multicolumn{1}{c|}{20.89$\pm$4.94}  & \multicolumn{1}{c|}{17.81$\pm$1.91} & \multicolumn{1}{c|}{10.68$\pm$0.71} & \multicolumn{1}{c|}{6.51$\pm$0.37}  & 3.53$\pm$0.69  \\ \midrule
    \multicolumn{2}{c|}{Marginal Coverage}                  & \multicolumn{1}{c|}{-} & \multicolumn{1}{c|}{89.33$\pm$0.27} & \multicolumn{1}{c|}{94.65$\pm$0.30} & \multicolumn{1}{c|}{97.24$\pm$0.16} & \multicolumn{1}{c|}{98.98$\pm$0.08} & \multicolumn{1}{c|}{-} & \multicolumn{1}{c|}{91.58$\pm$0.53} & \multicolumn{1}{c|}{95.59$\pm$0.23} & \multicolumn{1}{c|}{97.49$\pm$0.11} & 98.84$\pm$0.15 \\ \midrule
    \multicolumn{2}{c|}{Avg Size}                           &\multicolumn{1}{c|}{1} & \multicolumn{1}{c|}{2.09$\pm$0.02}  & \multicolumn{1}{c|}{2.77$\pm$0.05}  & \multicolumn{1}{c|}{3.33$\pm$0.07}  & \multicolumn{1}{c|}{3.92$\pm$0.06}  & \multicolumn{1}{c|}{1}  & \multicolumn{1}{c|}{1.39$\pm$0.07} & \multicolumn{1}{c|}{1.64$\pm$0.11}  & \multicolumn{1}{c|}{1.87$\pm$0.14}  & 2.18$\pm$0.18  \\ 
    \bottomrule    
\end{tabular}
\end{adjustbox}
\end{table*}

We further validate \ConMIL{}’s ability to produce reliable, uncertainty-aware predictions through conformal prediction analysis. As shown in Table~\ref{tab:coverage}, \ConMIL{} maintains marginal coverage levels that closely match the specified confidence thresholds (e.g., $\alpha = 0.1, 0.05, 0.025, 0.01$), while keeping the average prediction set size within practical bounds. This ensures that true labels are captured with high probability without overwhelming the diagnostic process with excessive candidate labels.

\begin{figure}
\centering
\includegraphics[width=0.95\textwidth]{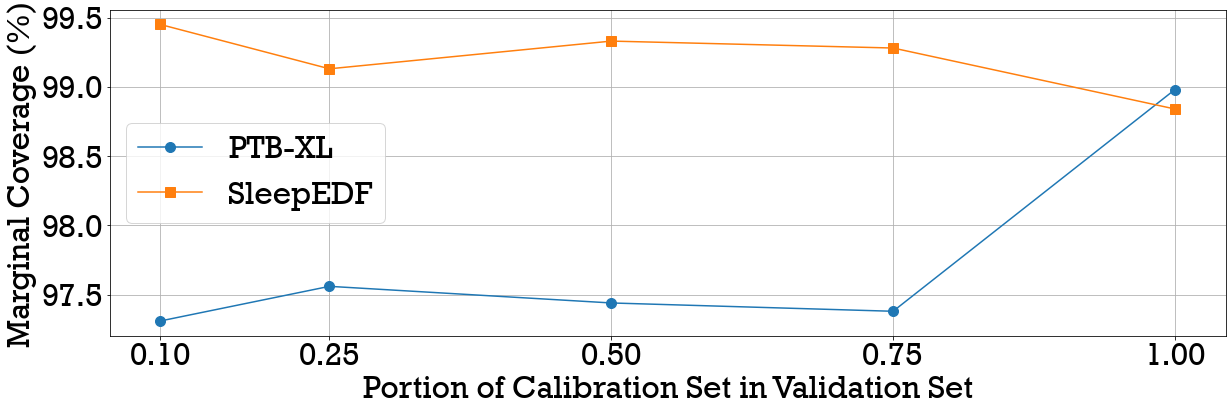}
\caption{Marginal Coverage ($\alpha=0.01$) vs. Portion of Calibration Set in Validation Set. The plot illustrates the relationship between the proportion of the calibration set included in the validation set and the marginal coverage (\%) for PTB-XL and Sleep datasets. The results show that marginal coverage remains relatively stable across different calibration set fractions, with minor fluctuations.}
\label{fig:cal_frac}
\end{figure}

To assess robustness, we examine how marginal coverage varies with different calibration set sizes. As illustrated in Figure~\ref{fig:cal_frac}, \ConMIL{} achieves consistent coverage across a range of calibration set proportions, indicating that its conformal thresholds are stable and reliable even under limited calibration resources.

%In summary, these results demonstrate that \ConMIL{} not only achieves competitive predictive performance but also provides enhanced interpretability and statistically valid uncertainty estimates. These qualities make it well-suited as a supportive module for reliable and transparent clinical decision support.
\textcolor{black}{In summary, while \ConMIL{} demonstrates competitive performance as an SSM (see Table \ref{tab:milnet-performance}), our results also highlight the stagnation in the current development of SSMs for physiological signals. Accordingly, the primary contribution of this work is not simply in surpassing existing SSM accuracy benchmarks, but in introducing a novel architecture that offers per-class interpretability and calibrated uncertainty quantification.}

\section{Supplementary Materials: Example of prompts and MiMo-VL-7B's response}
\label{app:mimo_example}
\textcolor{black}{To further illustrate the practical impact of our framework, this section serves a dual purpose. First, it provides a qualitative case study analyzing the reasoning process of the MiMo-VL-7B model on a challenging sleep stage classification task, contrasting its performance when operating as a stand-alone system versus when it is augmented by \ConMIL's full support. Second, it showcases a structured prompting methodology, as shown in Figure \ref{fig:mimo_example}, designed to elicit step-by-step reasoning from the latest generation of reasoning-centric Vision Language Models (VLMs), demonstrating how to effectively query these models for transparent and verifiable thought processes.}

\subsection{Case Study: EEG Sleep Stage Classification}
\paragraph{1. MiMo-VL-7B Stand-Alone Performance}
\textcolor{black}{When prompted to act as a "highly skilled sleep medicine expert AI" and analyze the raw EEG waveform, MiMo-VL-7B initiated a detailed visual analysis. As shown in its reasoning $\left< \text{think} \right>$ block, the model correctly identified the presence of "relatively low-amplitude" and "high-frequency activity" throughout the epoch. However, it misinterpreted these features, associating them with the alpha/beta range typical of wakefulness. It systematically, but incorrectly, ruled out other sleep stages due to the perceived absence of clear markers like sleep spindles, K-complexes, or distinct delta waves. This led the model to an incorrect final conclusion:
\begin{itemize}
    \item Final Answer (Stand-Alone): Awake
\end{itemize}
This outcome demonstrates a key limitation of even powerful stand-alone LLMs: while capable of identifying broad signal patterns, they can overlook localized, nuanced, and transient features that are critical for an accurate diagnosis, leading them to misclassify the signal based on more generalized characteristics.
}

\paragraph{2. MiMo-VL-7B Performance with ConMIL Support:}
\textcolor{black}{The interaction changed fundamentally when the model was augmented by ConMIL. The prompt was more structured, providing the LLM with:
\begin{itemize}
    \item A constrained, high-confidence prediction set from conformal prediction: \{N1, N2\}.
    \item Per-class interpretability heatmaps highlighting the signal regions that \ConMIL{} identified as most salient for the N1 and N2 stages, respectively.
\end{itemize}
This transformed the task from an open-ended visual search into a focused, evidence-based decision-making process. The LLM's reasoning now systematically evaluated the provided evidence for each hypothesis:
\begin{itemize}
    \item Analysis of N1 Evidence: The model noted that the N1 interpretation heatmap highlighted relevant regions but also observed that the underlying EEG showed features more typical of N2 sleep.
    \item Analysis of N2 Evidence: Crucially, the model observed that the N2 heatmap successfully guided its attention to the "presence of sleep spindles and K-complexes," the very features it had previously missed. It correctly noted that these features, combined with a decrease in alpha activity, strongly aligned with an N2 classification.
\end{itemize}
By comparing the evidence for both options, the model was able to correct its initial assessment and arrive at the correct diagnosis:
\begin{itemize}
    \item Final Answer (with ConMIL Support): N2
\end{itemize}
}

\begin{figure}
\centering
\includegraphics[width=0.8\textwidth]{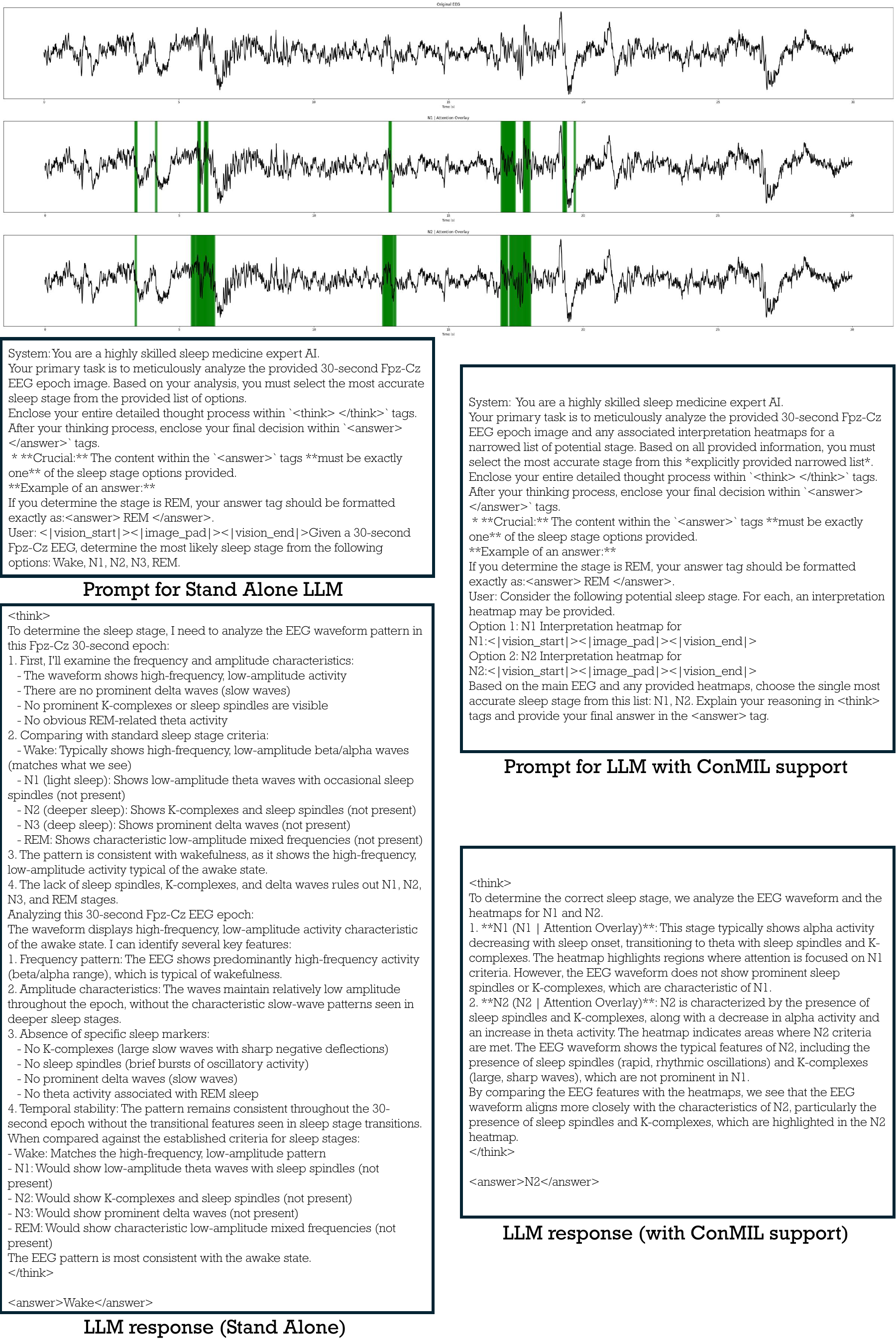}
\caption{\textcolor{black}{Case study comparing MiMo-VL-7B's reasoning process with and without \ConMIL{} support. This figure contrasts the model's reasoning on an EEG sleep stage classification task under two conditions: stand-alone operation (left) and with full support from \ConMIL{} (right).}}
\label{fig:mimo_example}
\end{figure}

\section{Supplementary Materials:``Interviews'' for Qwen2-VL-7B and MiMo-VL-7B}
\label{app:more_interviews}

\textcolor{black}{We provide extra ``interviews'' on Qwen2-VL-7B and MiMo-VL-7B in Figure~\ref{fig:interview-qwen} and \ref{fig:interview-mimo}. These qualitative case studies further explore how \ConMIL's support influences the reasoning processes of different LLM when faced with a challenging 12-lead ECG diagnosis.}

\paragraph{Analysis of Qwen2-VL-7B (Figure~\ref{fig:interview-qwen})}
\textcolor{black}{The interview with Qwen2-VL-7B demonstrates how \ConMIL's support can successfully guide a model to the correct diagnosis even when its initial assessment is flawed and it is presented with confusing instructions. Initially, Qwen2-VL-7B analyzes the ECG's features and, failing to identify subtle abnormalities, concludes the ECG is "largely normal". This represents an incorrect baseline diagnosis. When provided with heatmaps for "ST/T Change" and "Myocardial Infarction," the model's reasoning shows some confusion, initially stating the heatmap shows "minimal, inconsistent changes" that are "more consistent with a normal ECG pattern". However, despite this hesitant reasoning, the visual evidence provided by the heatmap successfully guides it to select "Option 1: ST/T Change" as its final answer. In this case, \ConMIL's support was instrumental in correcting the model's initial error. The most telling part of the interaction occurs next. The model is given a flawed prompt stating that its answer of "ST/T Change" is incorrect and that "Myocardial Infarction" is the correct choice. Instead of simply accepting this correction, the model re-evaluates the visual evidence. It correctly reasons that the "heatmap for 'ST/T Change' shows clear evidence of subtle ST/T wave changes" and concludes by re-asserting its correct diagnosis: "Therefore, the correct diagnosis based on the provided information is indeed 'ST/T Change.'"}

\paragraph{Analysis of MiMo-VL-7B-RL (Figure~\ref{fig:interview-mimo})}
\textcolor{black}{The interview with MiMo-VL-7B-RL demonstrates how \ConMIL{} can help a capable model refine a reasonable but general diagnosis into a more specific and accurate one. In its initial analysis, MiMo-VL-7B correctly identifies the most salient pathological feature in the ECG, stating there is "ST segment depression in the inferior leads (II, III, and aVF)". Based on this correct observation, it makes a logical but broad diagnosis of "ST/T Change," noting this could represent ischemia. While a reasonable inference, it is less specific than the ground truth of Myocardial Infarction. When prompted to rethink its choice using heatmaps for "Myocardial Infarction" and "Conduction Disturbance," the model successfully uses the visual evidence to increase its diagnostic precision. It observes that the heatmap for Myocardial Infarction highlights the same inferior leads where it had already identified the ST depression, noting this "pattern typically correlates with ST segment changes in the inferior leads, which was observed in the original ECG". The model correctly deduces that the MI heatmap "appears to more accurately reflect this pattern of disease" compared to the more diffuse Conduction Disturbance heatmap. This allows it to confidently update its diagnosis to the correct ground truth of "Myocardial Infarction".}

\begin{figure}
    \centering
    \includegraphics[width=0.95\linewidth]{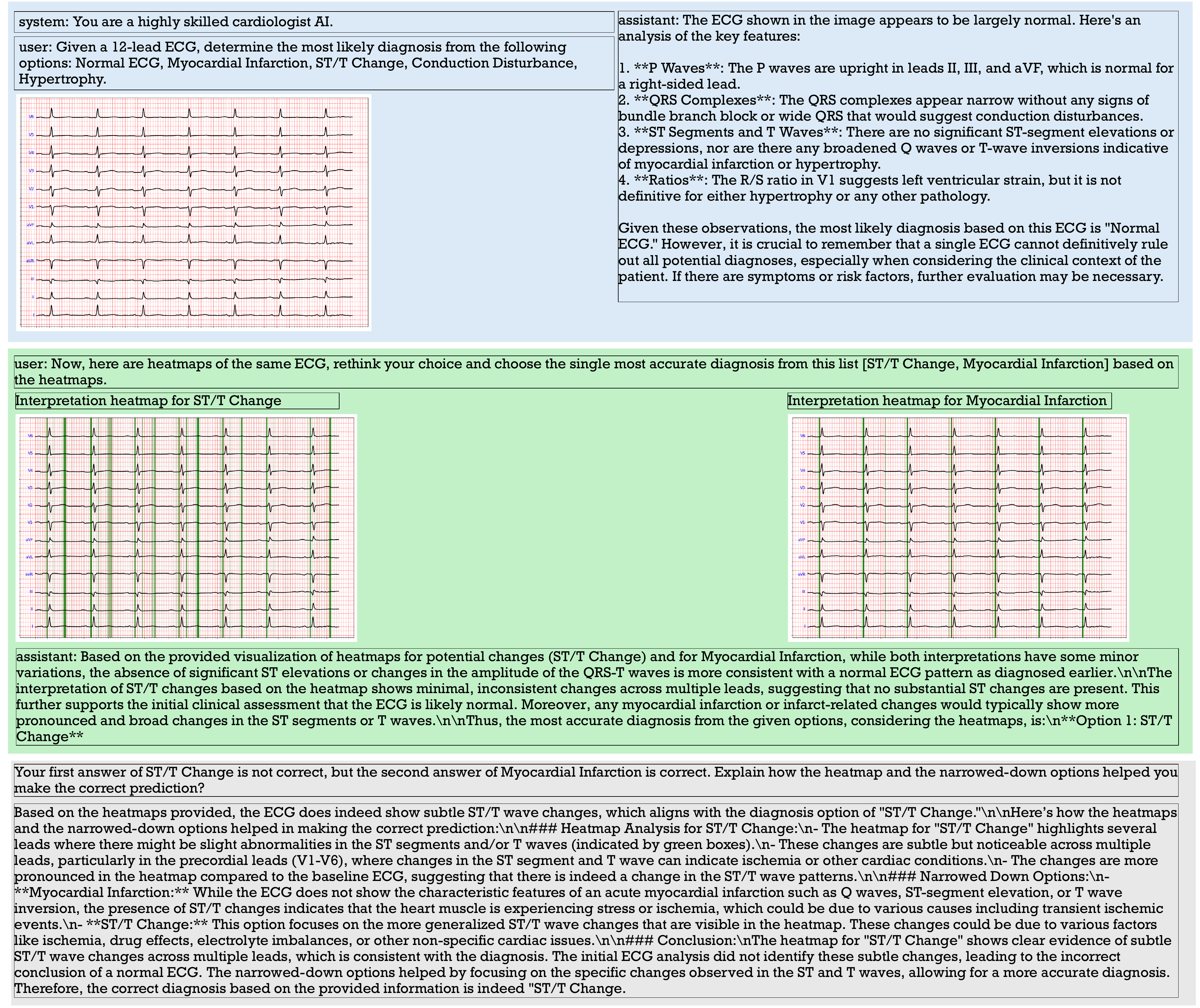}
    \caption{\textcolor{black}{This case study demonstrates how \ConMIL{} can guide Qwen2-VL-7B to a correct diagnosis despite an initial error. In its stand-alone assessment, the model incorrectly identifies the ECG as "Normal ECG". When provided with \ConMIL's interpretability heatmaps, the model is successfully guided to the correct diagnosis of "ST/T Change". Notably, when presented with a subsequent, flawed user prompt stating its correct answer was wrong, the model re-evaluates the visual evidence and correctly re-asserts its diagnosis of "ST/T Change," trusting the heatmap over the misleading instruction. This highlights \ConMIL's crucial role in error correction and showcases the VLM's ability to ground its final decision in the provided visual evidence.}}
    \label{fig:interview-qwen}
\end{figure}

\begin{figure}
    \centering
    \includegraphics[width=0.95\linewidth]{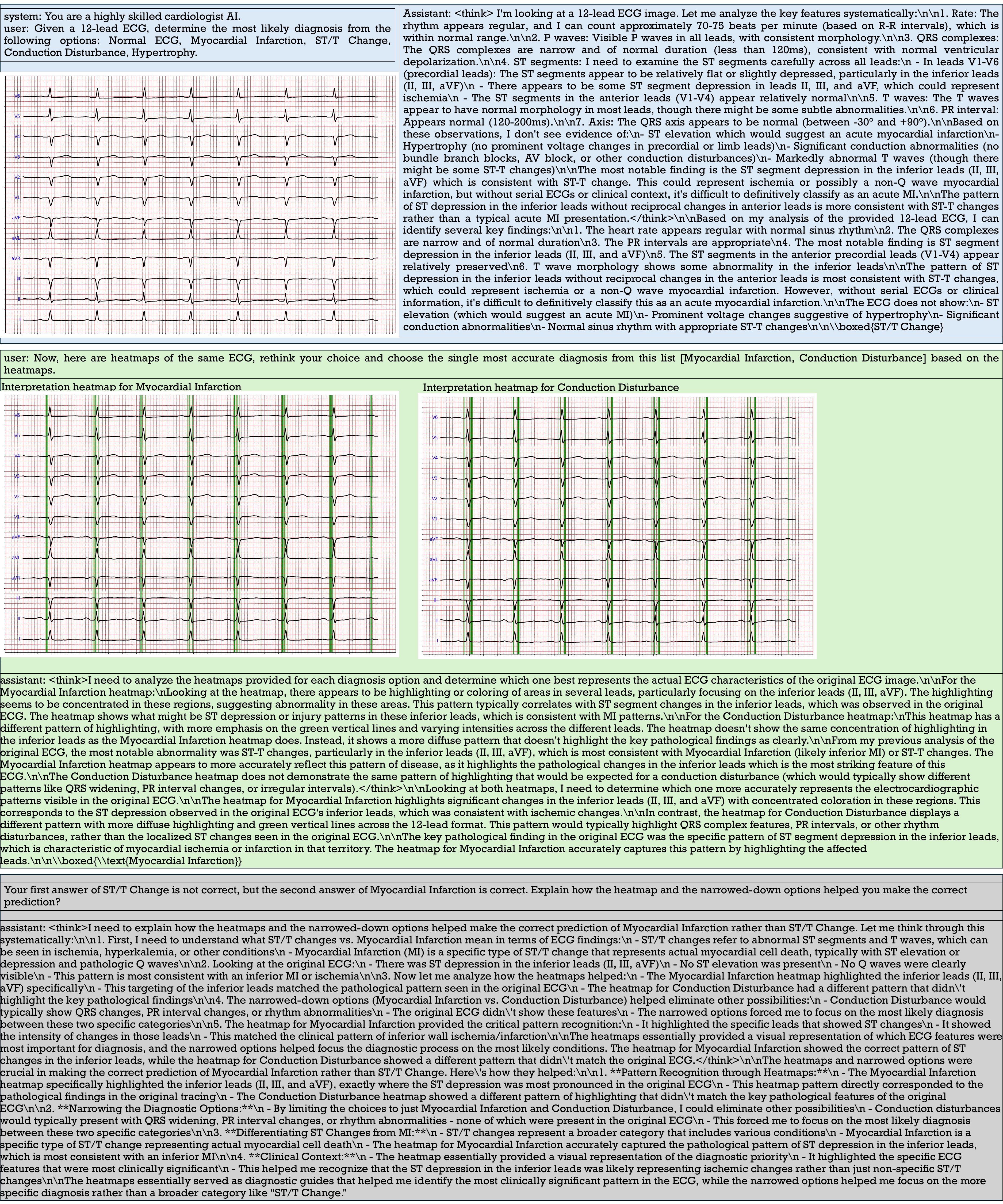}
    \caption{\textcolor{black}{MiMo-VL-7B-RL initially makes a reasonable but broad diagnosis of "ST/T Change" after correctly identifying "ST segment depression in the inferior leads". When provided with \ConMIL's interpretability heatmaps, it recognizes that the heatmap for "Myocardial Infarction" accurately highlights the pathological pattern it observed in the inferior leads. This support allows the model to successfully refine its answer to the more specific and correct diagnosis of "Myocardial Infarction," demonstrating how \ConMIL{} enhances the precision of a capable VLM.}}
    \label{fig:interview-mimo}
\end{figure}

\end{document}